\documentclass[twoside,11pt]{article}

\usepackage{jmlr2e}

\usepackage{graphicx}
\usepackage{subcaption}
\usepackage{booktabs} % for professional tables
\usepackage{adjustbox}

\usepackage{amsmath}
\usepackage{yhmath}
\usepackage{stmaryrd}
\usepackage{pgfplots}
\usepgfplotslibrary{colorbrewer}
\usepgfplotslibrary{fillbetween}
\usepgfplotslibrary{groupplots}
\usepackage{makecell}
\usepackage{comment}
\usepackage{todonotes}
\usepackage{bbm}
\usepackage{enumerate}

\definecolor{mypink1}{rgb}{0.858, 0.188, 0.478}
\definecolor{mypink2}{RGB}{219, 48, 122}
\definecolor{mypink3}{cmyk}{0, 0.7808, 0.4429, 0.1412}
\definecolor{mygray}{gray}{0.6}

\definecolor{mycolor8}{rgb}{0, 0, 1}
\definecolor{mycolor7}{rgb}{0.15, 0.15, 0.9}
\definecolor{mycolor6}{rgb}{0.3, 0.3, 0.75}
\definecolor{mycolor5}{rgb}{0.45, 0.45, 0.6}
\definecolor{mycolor4}{rgb}{0.6, 0.6, 0.45}
\definecolor{mycolor3}{rgb}{0.75, 0.75, 0.3}
\definecolor{mycolor2}{rgb}{0.9, 0.9, 0.15}
\definecolor{mycolor1}{rgb}{1, 1, 0}

\definecolor{mycolor8_}{rgb}{1, 0, 0}
\definecolor{mycolor7_}{rgb}{0.9, 0.15, 0.15}
\definecolor{mycolor6_}{rgb}{0.75, 0.3, 0.3}
\definecolor{mycolor5_}{rgb}{0.6, 0.45, 0.45}
\definecolor{mycolor4_}{rgb}{0.45, 0.6, 0.6}
\definecolor{mycolor3_}{rgb}{0.3, 0.75, 0.75}
\definecolor{mycolor2_}{rgb}{0.15, 0.9, 0.9}
\definecolor{mycolor1_}{rgb}{0, 1, 1}

\definecolor{mycolor8__}{rgb}{0, 1, 0}
\definecolor{mycolor7__}{rgb}{0.15, 0.9, 0.15}
\definecolor{mycolor6__}{rgb}{0.3, 0.75, 0.3}
\definecolor{mycolor5__}{rgb}{0.45, 0.6, 0.45}
\definecolor{mycolor4__}{rgb}{0.6, 0.45, 0.6}
\definecolor{mycolor3__}{rgb}{0.75, 0.3, 0.75}
\definecolor{mycolor2__}{rgb}{0.9, 0.15, 0.9}
\definecolor{mycolor1__}{rgb}{1, 0, 1}

\newcommand{\R}{\mathbb{R}}

\newcommand{\N}{\mathbb{N}}
\newcommand{\E}{\mathbb{E}}

\newcommand{\Ell}{\mathcal{L}}

\newcommand{\Prob}{\mathbb{P}}
\newcommand{\Id}{\mathcal{I}}

\newcommand{\limiting}[1]{\overset{\scriptscriptstyle\infty}{#1}}

\DeclareMathOperator{\tr}{tr}
\DeclareMathOperator{\sgn}{sgn}

\usepackage{lastpage}
\jmlrheading{23}{2022}{1-\pageref{LastPage}}{1/21; Revised 5/22}{9/22}{21-0000}{D\'avid Terj\'ek and Diego Gonz\'alez-S\'anchez}

\ShortHeadings{MLPs at the EOC: Concentration of the NTK}{Terj\'ek and Gonz\'alez-S\'anchez}
\firstpageno{1}

\begin{document}

\title{MLPs at the EOC: Concentration of the NTK}

\author{
  \name D\'avid Terj\'ek\thanks{Corresponding author.}
  \email dterjek@renyi.hu \\
  \addr Alfr\'ed R\'enyi Institute of Mathematics \\ Budapest, Hungary
  \AND
  \name Diego Gonz\'alez-S\'anchez 
  \email diegogs@renyi.hu \\
  \addr Alfr\'ed R\'enyi Institute of Mathematics \\ Budapest, Hungary
}

\editor{}

\maketitle

\begin{abstract}%   <- trailing '%' for backward compatibility of .sty file
We study the concentration of the Neural Tangent Kernel (NTK) $K_\theta : \R^{m_0} \times \R^{m_0} \to \R^{m_l \times m_l}$ of $l$-layer Multilayer Perceptrons (MLPs) $N : \R^{m_0} \times \Theta \to \R^{m_l}$ equipped with activation functions $\phi(s) = a s + b \vert s \vert$ for some $a,b \in \R$ with the parameter $\theta \in \Theta$ being initialized at the Edge Of Chaos (EOC). Without relying on the gradient independence assumption that has only been shown to hold asymptotically in the infinitely wide limit, we prove that an approximate version of gradient independence holds at finite width. Showing that the NTK entries $K_\theta(x_{i_1},x_{i_2})$ for $i_1,i_2 \in [1:n]$ over a dataset $\{x_1,\cdots,x_n\} \subset \R^{m_0}$ concentrate simultaneously via maximal inequalities, we prove that the NTK matrix $K(\theta) = [\frac{1}{n} K_\theta(x_{i_1},x_{i_2}) : i_1,i_2 \in [1:n]] \in \R^{nm_l \times nm_l}$ concentrates around its infinitely wide limit $\limiting{K} \in \R^{nm_l \times nm_l}$ without the need for linear overparameterization. Our results imply that in order to accurately approximate the limit, hidden layer widths have to grow quadratically as $m_k = k^2 m$ for some $m \in \N+1$ for sufficient concentration. For such MLPs, we obtain the concentration bound $\Prob( \Vert K(\theta) - \limiting{K} \Vert \leq O((\Delta_\phi^{-2} + m_l^{\frac{1}{2}} l) \kappa_\phi^2 m^{-\frac{1}{2}})) \geq 1-O(m^{-1})$ modulo logarithmic terms, where we denoted $\Delta_\phi = \frac{b^2}{a^2+b^2}$ and $\kappa_\phi = \frac{\vert a \vert + \vert b \vert}{\sqrt{a^2 + b^2}}$. This reveals in particular that the absolute value ($\Delta_\phi=1$, $\kappa_\phi=1$) beats the ReLU ($\Delta_\phi=\frac{1}{2}$, $\kappa_\phi=\sqrt{2}$) in terms of the concentration of the NTK.
\end{abstract}

%\begin{keywords}
%Neural Tangent Kernel, Gradient Independence Assumption, Condition number, Neural Tangent Parametrization, Maximal Update Parametrization.
%\end{keywords}

\section{Introduction} \label{introduction}
Formally introduced in the celebrated work of \citet{Jacotetal2018}, the NTK has been widely employed to analyze the problem of overparameterized learning. Given a neural network $N : \R^{m_0} \times \Theta \to \R^{m_l}$ that maps an input $x \in \R^{m_0}$ and a parameter $\theta \in \Theta$ to an output $N(x,\theta) \in \R^{m_l}$, the corresponding NTK at some parameter $\theta$ is the matrix-valued kernel $K_\theta : \R^{m_0} \times \R^{m_0} \to \R^{m_l \times m_l}$ defined as $K_\theta(x_1,x_2) = \partial_\theta N(x_1,\theta) {\partial_\theta N(x_2,\theta)}^*$ (the product of the Jacobian of $N(x_1,\cdot) : \Theta \to \R^{m_l}$ and the adjoint of the Jacobian of $N(x_2,\cdot) : \Theta \to \R^{m_l}$) for all input pairs $x_1,x_2 \in \R^{m_0}$. \citet{Jacotetal2018} showed that for MLPs using the Neural Tangent Parameterization (referred to as the NTP by \citet{Yangetal2021}), as width grows to infinity, $K_\theta$ at initialization (with $\theta$ drawn from the initial parameter distribution) converges in probability to a limiting NTK $\limiting{K}: \R^{m_0} \times \R^{m_0} \to \R^{m_l \times m_l}$. Later, \citet{Yang2020} proved almost sure convergence for a wide range of architectures while also giving theoretical justification to the gradient independence assumption (GIA) that was used heuristically by \citet{Jacotetal2018} to calculate $\limiting{K}$. Recently, \citet{Xuetal2024} proved that for the NTP, $K_\theta$ converges uniformly to $\limiting{K}$ when restricted to the hypersphere $\{ x_1, x_2 \in \R^k : \Vert x_1 \Vert = \Vert x_2 \Vert = 1 \}$, quantifying the convergence rate as well.

\citet{Jacotetal2018} proved that, in the infinitely wide limit, the NTK stays constant during gradient flow, which converges to a global minimum if the limiting NTK matrix $\limiting{K} = [\frac{1}{n} K_\infty(x_{i_1},x_{i_2}) : i_1,i_2 \in [1:n]]$ is positive definite. Then \citet{Duetal2018, Suetal2019, Oymaketal2019a, Aroraetal2019, Oymaketal2020, Songetal2021, Duetal2019, Zouetal2019, Nguyenetal2020, Nguyen2021, Liuetal2022} used similar ideas to prove that training finite width MLPs with gradient descent on a dataset $\{x_1,\cdots,x_n\} \in \R^{m_0}$ converges globally as long as the NTK matrix $K(\theta)$ stays positive during training. Using the NTP, the so-called lazy training phenomenon \citep{Chizatetal2019} can be exploited to show that even though $K(\theta)$ does not stay constant, as the width increases, it changes less and less during gradient descent, so that as long as the smallest eigenvalue of $K(\theta)$ is positive at initialization, it stays positive during training with sufficient overparameterization. Inspired by this, many works \citep{Montanarietal2020, Nguyenetal2021, Wangetal2021, Bombarietal2022, Banerjeeetal2023} started studying the concentration of the smallest eigenvalue of the NTK at initialization.

\citet{Woodworthetal2020} identified the so-called kernel and rich regimes of neural networks, with the NTP being a prime example of an MLP belonging to the kernel regime. In the kernel regime, lazy training makes wide models behave as random feature models, while in the rich regime, this phenomenon is absent. \citet{Yangetal2021} showed that in the kernel regime, feature learning does not happen in the sense that hidden layer activations are almost constant during training. \citet{Yangetal2021} proposed an MLP parameterization called the Maximal Update Parameterization ($\mu$P) that, being in the rich regime, does admit feature learning, even in the infinitely wide limit. Unfortunately, while the convergence of gradient descent in overparameterized learning in the kernel regime is well understood, much less is known in the rich regime, where the NTK evolves during training in a nontrivial manner. Nevertheless, in both the rich and kernel regimes, the behavior of $K(\theta)$ at initialization seems to play an important role in understanding gradient descent.

Parallel to these developments, the study of infinitely deep neural networks led \citet{Pooleetal2016} to the discovery of the so-called Edge of Chaos (EOC). \citet{Schoenholzetal2017} showed that the EOC is the regime where infinitely deep MLPs avoid both exploding and vanishing gradients. In this regime, \citet{Hayouetal2019} described the asymptotic behavior of the cosines (correlations) of the activations in the infinitely wide limit, \citet{Xiaoetal2020} characterized the spectrum of $K_\infty$ by sending first the width and then depth to infinity, \citet{Hayouetal2022} quantified the entries of $K_\infty$ and \citet{Seleznovaetal22} studied the entries of both $K_\theta$ and $K_\infty$ when width and depth grow with a constant ratio. Additionally, using the NTP as width and depth tend to infinity together, \citet{Haninetal2020} proved that the NTK does not become constant in the limit. Recently, \citet{Yangetal2024} extended $\mu$P to infinitely deep residual networks, identifying feature diversity (measuring the difference of activations that are in close proximity across depth) as an essential factor in deep neural networks (similar to feature learning in wide ones), showing in particular that the absolute value $\vert\cdot\vert$ maximizes feature diversity among homogeneous activation functions.

The motivation for our work was to study the concentration of the NTK matrix $K(\theta)$ around the limiting NTK matrix $\limiting{K}$ at initialization with an MLP parameterization that can exemplify both the kernel and rich regimes, equipped with $(a,b)$-ReLU activations $\phi(s) = a s + b \vert s \vert$ and varying layer widths, quantifying the effects of such hyperparameters.

We start with introducing a general MLP parameterization whose hyperparameters include varying layer widths, scaling coefficients (controlling kernel and rich regime behavior) and vector-valued output. Then we show that $K_\theta$ concentrates around its expectation with respect to the last layer matrix, which decomposes as a layerwise sum of products of inner products of activations and Frobenius inner products of backpropagation matrices. The terms in the sum are weighted based on the scaling coefficients, leading to an optimal choice of scaling coefficients \eqref{eq:optimal_qs} ensuring that none of the terms will vanish or blow up, with the hyperparameter $q \in \R$ interpolating between the kernel regime at $q=0$ and the rich regime at $q=1$. We then focus on the layerwise concentration of the components. Instead of treating the activation inner products directly, we study the concentration of the activation norms and of proxies of the cosine distances of activations, which by the law of cosines will yield the optimal concentration error of the activation cosines. Computing the expectation of the backpropagation inner products is usually done by heuristically relying on the GIA, which has only been rigorously justified asymptotically in the infinitely wide limit by \citep{Yang2020}. Avoiding the GIA heuristic, we prove that an approximate form of gradient independence holds for finite width, quantifying the rate at which the gradient dependence error term vanishes. In particular, we find that the strength of gradient dependence depends on the activation cosines, the propagation of which is quantified exactly in the infinitely wide limit at the EOC by \citet{mlpsateoc1}. Employing these results, we show that the components concentrate simultaneously for all layers in the MLP over a dataset, with the concentration increasing only logarithmically in terms of depth provided that hidden layer sizes grow quadratically \eqref{eq:optimal_gammas} as $m_k = k^2 m$ for a hyperparameter $m \in \N+1$. Note that we restrict to this setting only in our results concerning simultaneous concentration, enabling the reader to prove analogous concentration bounds for other layer width patterns. We argue that we argue that this quadratic growth is not only sufficient but necessary in order to accurately approximate the infinitely wide limit. With these in hand, we prove our main result about the concentration of the NTK matrix around its limit, stated below in a slightly simplified form.

\begin{theorem}[Limiting concentration of $K(\theta)$ (simplified)]~\\
Given the MLP $N: \R^{m_0} \times \Theta \to \R^{m_l}$ defined in \S~\ref{mlp}, a dataset $\{x_1,\cdots,x_n\} \subset \R^{m_0}$ of size $n \in \N+2$ with no parallel data points and setting \eqref{eq:optimal_qs} and \eqref{eq:optimal_gammas}, we have that
\[
\Prob\left( \left\Vert K(\theta) - \limiting{K} \right\Vert \leq O\left( \overline{\tau}^2 \left( \Delta_\phi^{-2} + \left( \log(l) + m_l^{\frac{1}{2}} \right) l \right) \sqrt{\log(ln) \log(m)} \kappa_\phi^2 m^{-\frac{1}{2}} \right) \right)
\]
is at least $1-O(m^{-1})$ with $\overline{\tau} = \max_{i \in [1:n]}\left\{ \Vert x_i \Vert \right\}$.
\end{theorem}

Note that any dataset $\{x_1,\cdots,x_n\}$ with no repeated data points can be turned into one with no parallel data points by replacing $x_i$ for all $i \in [1:n]$ with $[x_i,\beta]$ for some $\beta>0$, which is equivalent to having a bias in the first layer. The above result can be combined with \citet[Theorem~18]{mlpsateoc1} to obtain spectral bounds for $K(\theta)$ at initialization. Denoting the iterates $\theta_t$ for $t \in \N+1$ obtained from the initial parameter $\theta$ by performing gradient descent on some loss function over the dataset, these spectral bounds should be sufficient to prove the convergence of gradient descent in the kernel regime where $\Vert K(\theta_t) - K(\theta) \Vert$ can be shown to vanish in terms of $m$. Unfortunately, more is needed in the rich regime, where $K(\theta_t)$ deviates significantly from the initial $K(\theta)$ in the absence of lazy training. Understanding the nature of these deviations can be the key to understanding the global convergence of gradient descent in the presence of feature learning. Note that the hyperparameter $q$ interpolating kernel ($q=0$) and rich regime ($q=1$) behavior does not appear in the theorem above, as these options result in identical NTKs at initialization. 

The organization of the rest of the paper is as follows. We conclude \S~\ref{introduction} by discussing related works in \S~\ref{related} and listing our contributions in \S~\ref{contributions} and introduce some notation in \S~\ref{preliminaries}. In \S~\ref{ntk}, we propose our general MLP formulation in \S~\ref{mlp} and derive its Jacobian, study layerwise concentration of the NTK components in \S~\ref{layerwise} and then prove the simultaneous concentration of all components and the NTK matrix itself over a dataset in \S~\ref{limiting}. We conclude by discussing the limitations of our work in \S~\ref{limitations} along with future directions.

\subsection{Related work}\label{related}

\citet{Duetal2018} and \citet{Suetal2019} proved that the term of the NTK matrix corresponding to the first layer matrix concentrates around its limit for shallow ($l=2$) ReLU MLPs using the NTP. \citet{Duetal2019} proved that the term of the NTK matrix corresponding to the second-to-last layer matrix concentrates around its limit for deep MLPs with hidden layers of the same size and smooth activation functions using the NTP. Recently, for deep ReLU MLPs with hidden layers of the same size using the NTP, \citet{Xuetal2024} proved that all terms of the NTK except the one corresponding to the last layer matrix uniformly concentrate around those in the limiting NTK for data from the unit sphere, i.e., all terms except the last in $K_\theta(x_1,x_2)$ concentrate around those in $\limiting{K}(x_1,x_2)$ for all $x_1,x_2 \in \R^{m_0}$ with $\Vert x_1 \Vert = \Vert x_2 \Vert = 1$. While this concentration bound can turn into a bound for $\Vert K(\theta)-\limiting{K} \Vert$ for spherical datasets of any size, their proof relies heavily on the fact that the number of possible activation patterns for the ReLU is finite, making it unlikely to generalize to nonhomogeneous activations. Additionally, the amount of overparameterization required in terms of the number of hidden layers grows much faster than ours as \citet{Xuetal2024} need $m = \Omega(e^{(l-1)^2})$. These works do not treat the last NTK term because they keep the output layer matrix fixed, making the last term absent in their formulation. In contrast, we consider the realistic setting with all layer matrices including the last one being random. On top of this, while \citet{Duetal2018, Suetal2019, Duetal2019, Xuetal2024} use the NTP, we study a general parameterization that covers both the kernel and the rich regimes.

Many works, including \citet{Jacotetal2018}, made implicit use of the GIA heuristic before it was justified on a theoretical basis by \citet{Yang2020}, extended in \citet{Yang2021} to cover a wider range of scenarios using free probability. These works show that gradient independence holds with very general assumptions for a wide range of architectures asymptotically in the infinite width limit, retroactively validating the calculations of \citet{Jacotetal2018} that led to the limiting NTK. Since we consider MLPs of finite width, we cannot rely on the asymptotic theory of \citet{Yang2020, Yang2021}. Instead, we quantify the error resulting from gradient dependence at finite width, showing that it vanishes at the rate $O(m^{-1})$.

\citet{Yangetal2021} proposed $\mu$P focusing on neural networks with constant hidden layer sizes and later extended it to varying layer widths by \citet{Yangetal2023} in what is known as the Spectral Parameterization (SP). While our MLP parameterization in \S~\ref{mlp} is another such extension of $\mu$P, it does not cover SP. One property of the latter is that the norms of hidden layer activations scale as the square roots of hidden layers by \citet[Desideratum~1]{Yangetal2023}, which means that SP is not at the EOC, where the activation norms across depth are approximately equal to the norm of the input for homogeneous activations at the EOC by \citet[\S~3.1]{Hayouetal2019}. This makes the corresponding limiting NTK dependent on the relative sizes of hidden layers. On the contrary, in our parameterization there is no such dependence, with the hidden layer sizes serving only to control the strength of concentration in the individual layers.

\subsection{Contributions} \label{contributions}
We propose
\begin{itemize}
\item an MLP parameterization with $(a,b)$-ReLUs at the EOC exemplifying both the kernel and rich regimes,
\item a width pattern enabling the accurate approximation of the infinitely wide limit and
\item a fully quantitative bound for the concentration of the NTK matrix around its limit.
\end{itemize}

\section{Preliminaries}\label{preliminaries}

Given $i, j \in \N$, we define the tuple $[i:j] = (i,i+1,\cdots,j-1,j)$ (which is the empty tuple $()$ if $i > j$). For any $m, n \in \N$, we denote by $m\N+n$ the set $\{mr+n : r \in \N\}$. We denote by $\Vert \cdot \Vert$ the Euclidean and by $\Vert \cdot \Vert_\infty$ the max norm on $\R^n$. Let $G,H$ be Hilbert spaces. The space of bounded linear operators from $G$ to $H$ is denoted $\Ell(G,H)$ and we equip it with the operator norm $\Vert\cdot\Vert$. The adjoint of a linear operator $A \in \Ell(G,H)$ is the unique linear operator $A^* \in \Ell(H,G)$ such that $\langle A x_1, x_2 \rangle = \langle x_1, A^* x_2 \rangle$ for all $x_1 \in G$ and $x_2 \in H$. For Euclidean spaces $G=\R^m$, $H=\R^n$, we denote the space of $n \times m$ matrices $\R^{n \times m} = \Ell(H,G)$. For such matrices, we denote the Frobenius norm by $\Vert \cdot \Vert_F$ and the infinity norm by $\Vert \cdot \Vert_\infty$ (with the latter defined as $\Vert A \Vert_\infty = \max_{i \in [1:n]}\{ \sum_{j \in [1:m]} \vert A_{i,j} \vert \}$). We denote the set of $n \times n$ symmetric matrices by $\mathbb{S}^n = \{ A \in \R^{n \times n} : A = A^*\}$ and the set of $n \times n$ symmetric positive semidefinite matrices by $\mathbb{S}^n_+ = \{ A \in \mathbb{S}^n : \langle x, A x \rangle \geq 0 \text{ for } \forall x \in \R^n \}$. For $A \in \mathbb{S}^n$, we denote the $i$th eigenvalue by $\lambda_i(A)$ with the order being descending as $\lambda_1(A) \geq \cdots \geq \lambda_n(A)$ and the smallest and largest eigenvalues by $\lambda_{\min}(A)=\lambda_n(A)$ and $\lambda_{\max}(A)=\lambda_1(A)=\Vert A \Vert$, respectively. Note that by the Gershgorin circle theorem we have $\|A\| \le \|A\|_\infty$ for any $A \in \mathbb{S}^n$. We denote by $\Id_n \in \R^{n \times n}$ the identity matrix on $\R^n$. We denote the tensor product of a pair of vectors $x,y \in \R^n$ by $x \otimes y = [ x_{i_1} y_{i_2} : i_1,i_2 \in [1:n]] \in \R^{n\times n}$ and the second tensor power of a vector $x \in \R^n$ by $x^{\otimes 2} = x \otimes x\in \mathbb{S}^n_+$. For $n \in \N+1$, we denote the $n$-dimensional constant $1$ vector by $\mathbbm{1}_n = [ 1 : i \in [1:n]] \in \R^n$. For matrices $A_1 \in \R^{n_1 \times m_1}$ and $A_2 \in \R^{n_2 \times m_2}$, we denote their Kronecker product $A_1 \boxtimes A_2 = [ {A_1}_{i_1,i_2} A_2 : i_1 \in [1:n_1], i_2 \in [1:m_1]] \in \R^{n_1 n_2 \times m_1 m_2}$. Given $x \in \R^n$, we define the corresponding diagonal matrix $D_x \in \mathbb{S}^n$ as ${D_x}_{i_1,i_2} = x_i$ if $i_1 = i_2 = i$ and $0$ otherwise for all $i_1,i_2 \in [1:n]$. Given $m, n \in \N+1$ and $x \in \R^m$, we define the right multiplier operator $M_{x,n} \in \Ell(\R^{n \times m}, \R^n)$ as $M_{x,n} A = A x$ for all $A \in \R^{n \times m}$. Note that $\Vert M_{x,n} \Vert \leq \Vert x \Vert$ (i.e., the operator norm of $M_{x,n}$ is bounded by the Euclidean norm of $x$) and the adjoint $M_{x_1,n}^* \in \Ell(\R^n, \R^{n \times m})$ is given as $M_{x_1,n}^* x_2 = x_2 \otimes x_1$ for all $x_1 \in \R^m$ and $x_2 \in \R^n$, implying in particular that $M_{x_1,n} M_{x_2,n}^* = \langle x_1, x_2 \rangle \Id_n$ for all $x_1,x_2 \in \R^m$.

The infinity and Lipschitz norms of real-valued functions are denoted by $\Vert \cdot \Vert_\infty$ and $\Vert \cdot \Vert_L$, respectively. Given a function $F:G \to H$, we say that it is differentiable if it is Fr\'echet differentiable, i.e., if there exists a bounded linear operator $\partial F(x) \in \Ell(G,H)$, which we refer to as the Jacobian of $F$ at $x$, satisfying $\lim_{y \to x }\frac{\Vert F(y) - F(x) - \partial F(x) (y - x) \Vert}{\Vert y - x \Vert}=0$. For a function $f$ with the same domain and codomain, we denote by $f^{\circ n}$ the nested composition of $f$ with itself $n \in \N$ times, with $f^{\circ 0}$ being the identity. We use the $O(\cdot)$ and $\Omega(\cdot)$ asymptotic notation in the sense that for functions $f,g : \N \to \R_+$, we say that $f(m) = O(g(m))$ (resp. $f(m) = \Omega(g(m))$) if there exists implicit constants $C \in \R_+$ and $m_0 \in \N$ such that $f(m) \leq C g(m)$ (resp. $f(m) \geq C g(m)$) for all $m \geq m_0$. The notation $f = \Theta(g)$ means that both $f=O(g)$ and $f=\Omega(g)$ hold.

A real-valued random variable $X$ is $K$-sub-gaussian if its sub-gaussian norm $\Vert X \Vert_{\psi_2} = \inf\left\{ t > 0 : \E e^{\frac{X^2}{t^2}} \leq 2 \right\}$ satisfies the bound $\Vert X \Vert_{\psi_2} \leq K$ and $K$-sub-exponential if its sub-exponential norm $\Vert X \Vert_{\psi_1} = \inf\left\{ t > 0 : \E e^{\frac{\vert X \vert}{t}} \leq 2 \right\}$ satisfies the bound $\Vert X \Vert_{\psi_1} \leq K$. An $\R^n$-valued random vector $X$ is $K$-sub-gaussian if the real-valued random variable $\langle X, x \rangle$ is $K$-sub-gaussian for all vectors $x \in \R^n$ such that $\Vert x \Vert=1$. A $K$-sub-gaussian $X$ concentrates as $\Prob(\vert X \vert \geq t) \leq 2e^{-\frac{t^2}{O(K)^2}}$ for all $t \geq 0$ and a $K$-sub-exponential $X$ concentrates as $\Prob(\vert X \vert \geq t) \leq 2e^{-\frac{t}{O(K)}}$ for all $t \geq 0$. More details on this subject can be found in \citet{Vershynin2018}, which is our main reference in this work.

Given $\mu \in \R^n$ and $\Sigma \in \mathbb{S}^n_+$, we denote by $\mathcal{N}(\mu,\Sigma)$ the multivariate Gaussian distribution with mean $\mu$ and covariance $\Sigma$. In particular, $\mathcal{N}(0,1)$ is the standard Gaussian distribution. By $X \sim \mathcal{N}(\mu,\Sigma)$ we mean that the random vector $X$ is distributed according to $\mathcal{N}(\mu,\Sigma)$. We use the same notation to denote the corresponding probability measure, i.e., $\E_{X \sim \mathcal{N}(\mu,\Sigma)} f(X) = \int f d\mathcal{N}(\mu,\Sigma) = \int f(x) d\mathcal{N}(x \vert \mu,\Sigma)$. We denote the norm of the Hilbert space $L^2(\mathcal{N}(0,1))$ by $\Vert f \Vert_{\mathcal{N}(0,1)} = \sqrt{\int f^2 d\mathcal{N}(0,1)}$ for $f \in L^2(\mathcal{N}(0,1))$.

\section{NTK at the EOC}\label{ntk}

In the following subsections, we first introduce our MLP parameterization and derive its Jacobian, then analyze the layerwise concentration of the components of its NTK and finally prove the simultaneous concentration of the NTK components and the NTK matrix itself over a dataset.

\subsection{Multilayer Perceptron}\label{mlp}

We introduce the MLP formulation which will be the focus of our analysis. Let $l \in \N+2$ be the depth, $\R^{m_0}$ the input space and $\Theta = \Theta_{1:l} = \prod_{k=1}^l \Theta_k$ the parameter space with parameter subspaces $\Theta_k = \R^{m_k \times m_{k-1}}$, input dimension $m_0 \in \N+1$, hidden layer widths $m_k = \gamma_k m$ for $k \in [1:l-1]$ for width parameters $m \in \N+1$ and $\gamma_k \in \N+1$ for $k \in [1:l-1]$ and output dimension $m_l \in \N+1$. We denote parameters as $\theta = \theta_{1:l} = [A_k : k \in [1:l]] \in \Theta$ with layer matrices $A_k \in \Theta_k$. Let $q_k \in \R$ for $k \in [1:l]$ be the scaling coefficients. Finally, let $a,b \in \R$ and $\phi : \R \to \R$ be the $(a,b)$-ReLU as defined below, which is going to be the activation function. We initialize the matrices $A_k \sim \mathcal{N}( 0,\sigma^2 m^{-q_k} \Id_{m_k \times m_{k-1}} )$ for $k \in [1:l]$ with $\sigma = (a^2 + b^2)^{-\frac{1}{2}}$ to ensure that the MLP is at the EOC by \citet[Lemma~3]{Hayouetal2019}. The corresponding probability space is the triple $(\Theta, \mathcal{B}(\Theta), \Prob) = (\Theta_1 \times \cdots \times \Theta_l, \mathcal{B}(\Theta_1) \otimes \cdots \otimes \mathcal{B}(\Theta_l), \Prob_1 \otimes \cdots \otimes \Prob_l) = (\Theta_1, \mathcal{B}(\Theta_1), \Prob_1) \otimes \cdots \otimes (\Theta_l, \mathcal{B}(\Theta_l),\Prob_l)$, which is the product of the individual probability spaces corresponding to each layer. The individual expectations are denoted as $\E_{A_k} X(\theta_{1:k-1},A_k) = \int X(\theta_{1:k-1},A_k) d\Prob_k(A_k)$ for any random variable $X : \Theta_{1:k} \to \R$ and (sub)parameter $\theta_{1:k-1} \in \Theta_{1:k-1}$.

\begin{definition}[$(a,b)$-ReLU]~\\
Given $a,b \in \R$, define the $(a,b)$-ReLU $\phi : \R \to \R$ for all $s \in \R$ as $\phi(s) = as + b\vert s \vert$, so that $\phi'(s) = a + b \sgn(s)$ for all $s \in \R \setminus \{ 0 \}$ and $\Vert \phi \Vert_L = \vert a \vert + \vert b \vert$.
\end{definition}
Unless $b=0$, $\phi$ is not differentiable at $s=0$ in the usual sense, but any function $\psi : \R \to \R$ such that $\psi(s) = a + b \sgn(s)$ for all $s \in \R \setminus \{ 0 \}$ and $\psi(0) \in [a-b,a+b]$ can serve as its derivative in some suitable generalized sense. By abuse of notation, we define $\phi': \R \to \R$ as $\phi'(s) = a + b \sgn(s)$ for all $s \in \R$, so that $\phi'(0) = a$.

Define an $l$-layer MLP $N : \R^{m_0} \times \Theta \to \R^{m_l}$ for any $x \in \R^{m_0}$ and $\theta \in \Theta$ recursively as 
\[
N(x,\theta) = A_l m_{l-1}^{-\frac{1}{2}} \phi(N_{l-1}(x,\theta_{1 : l-1}))
\]
with the input layer $N_1 : \R^{m_0} \times \Theta_1 \to \R^{m_1}$ defined as $N_1(x,\theta_1) = m^{\frac{q_1}{2}} A_1 x$ and the hidden layers $N_k : \R^{m_0} \times \Theta_{1:k} \to \R^{m_k}$ for $k \in [2:l-1]$ defined as
\[
N_k(x,\theta_{1:k}) = m^{\frac{q_k}{2}} A_k m_{k-1}^{-\frac{1}{2}} \phi(N_{k-1}(x,\theta_{1 : k-1})).
\]
For an input $x$ and a parameter $\theta$, denote the activations by $x_1(x) = x \in \R^{m_0}$ and $x_k(x, \theta_{1 : k-1}) = m_{k-1}^{-\frac{1}{2}} \phi(N_{k-1}(x, \theta_{1 : k-1})) \in \R^{m_{k-1}}$ for $k \in [2:l]$ and the derivatives of the activations\footnote{Note that the naming is informal, but we do have that the vector $x_k'(x,\theta_{1 : k-1})$ is the diagonal of the Jacobian matrix $\partial_{N_{k-1}(x,\theta_{1 : k-1})} x_k(x, \theta_{1 : k-1})$.} by $x_k'(x,\theta_{1 : k-1}) = m_{k-1}^{-\frac{1}{2}} \phi'(N_{k-1}(x,\theta_{1 : k-1})) \in \R^{m_{k-1}}$ for $k \in [2:l]$. We can then write the forward pass in a compact manner as $N_k(x,\theta_{1:k}) = m^{\frac{q_k}{2}} A_k x_k(x,\theta_{1:k-1})$ for $k \in [1:l-1]$ and $N(x,\theta) = A_l x_l(x,\theta_{1:l-1})$.

\begin{remark}[Relation to other parameterizations]~\\
The NTK paramerization (NTP) of \citet{Jacotetal2018} is recovered by setting $q_1=\cdots=q_l=0$, while the Maximal Update Parameterization ($\mu$P) of \citet{Yangetal2022} corresponds to the case $q_1=\cdots=q_l=1$ and $\gamma_1 = \cdots = \gamma_{l-1} = 1$.
\end{remark}

We will describe the Jacobian of the neural network mapping inductively as follows. Note that for the first layer, as $N_1(x,\theta_1)=m^{\frac{q_1}{2}} A_1 x_1(x)$ is linear in $\theta_1 = A_1$ its Jacobian is itself, meaning that if $\theta_1' = A_1' \in \Theta_1$, then $\partial_{\theta_1} N_1(x,\theta_1) \theta_1' = m^{\frac{q_1}{2}} A_1' x_1(x)$. For convenience, we will write that $\partial_{\theta_1} N_1(x,\theta_1) = m^{\frac{q_1}{2}} M_{x_1(x),m_1} \in \Ell(\Theta_1, \R^{m_1})$. Via the chain rule, it follows that the Jacobian $\partial_{\theta_{1:k}} N_k(x,\theta_{1:k}) \in \Ell(\Theta_{1:k}, \R^{m_k})$ for the $k$th layer is
\[
\partial_{\theta_{1:k}} N_k(x,\theta_{1:k}) = m^{\frac{q_k}{2}} \left[ \begin{array}{cc} A_k D_{x_k'(x, \theta_{1:k-1})} \partial_{\theta_{1 : k-1}} N_{k-1}(x,\theta_{1 : k-1}) & M_{x_k(x, \theta_{1:k-1}),m_k} \end{array} \right],
\]
understood as a block matrix to be multiplied by a block vector of the form $\left[ \begin{array}{c} \theta_{1:k-1}' \\ A_k' \end{array} \right] \in \Theta_{1:k}$. The full Jacobian $\partial_\theta N(x,\theta) \in \Ell(\Theta, \R^{m_l})$ equals
\begin{equation}\label{eq:jacobian}
\partial_\theta N(x,\theta) = \left[ \begin{array}{cc} A_l D_{x_l'(x, \theta_{1:l-1})} \partial_{\theta_{1 : l-1}} N_{l-1}(x,\theta_{1 : l-1}) & M_{x_l(x, \theta_{1:l-1}),m_l} \end{array} \right].
\end{equation}

\subsection{Layerwise Concentration of the NTK}\label{layerwise}

In this section, we decompose the NTK of the MLP introduced in \S~\ref{mlp} and analyze the concentration of its components with respect to the individual layer matrices.

\begin{definition}[Neural Tangent Kernel]\label{def:ntk_kernel}~\\
Given the MLP $N: \R^{m_0} \times \Theta \to \R^{m_l}$ defined in \S~\ref{mlp} and a parameter $\theta \in \Theta$, the corresponding NTK $K_\theta : \R^{m_0} \times \R^{m_0} \to \R^{m_l \times m_l}$ is the matrix-valued kernel defined as 
\[
K_\theta(x_1,x_2) = \partial_\theta N(x_1,\theta) {\partial_\theta N(x_2,\theta)}^*
\]
for all $x_1, x_2 \in \R^{m_0}$.
\end{definition}

For convenience, we denote the norms of the activations as $\tau_k(x,\theta_{1:k-1}) = \Vert x_k(x, \theta_{1 : k-1}) \Vert$ for $x \in \R^{m_0}$, $k \in [1:l]$ and $\theta_{1 : k-1} \in \Theta_{1 : k-1}$, the inner products of the activations as $X_k(x_1,x_2,\theta_{1:k-1}) = \langle x_k(x_1, \theta_{1 : k-1}), x_k(x_2, \theta_{1 : k-1}) \rangle$ for $x_1,x_2 \in \R^{m_0}$, $k \in [1:l]$ and $\theta_{1 : k-1} \in \Theta_{1 : k-1}$ and the cosines of the activations as
\[
\rho_k(x_1,x_2,\theta_{1:k-1}) = \left\langle \frac{x_k(x_1, \theta_{1 : k-1})}{\Vert x_k(x_1, \theta_{1 : k-1}) \Vert}, \frac{x_k(x_2, \theta_{1 : k-1})}{\Vert x_k(x_2, \theta_{1 : k-1}) \Vert} \right\rangle \in [-1,1]
\]
for $x_1,x_2 \in \R^{m_0}$, $k \in [1:l]$ and $\theta_{1 : k-1} \in \Theta_{1 : k-1}$.

\begin{definition}[Backpropagation matrices]\label{def:backprop_matrix}~\\
Given $x \in \R^{m_0}$, $k_1 \leq k_2 \in [2:l]$ and $\theta \in \Theta$, define the backpropagation matrix
\[
B_{k_1, k_2}(x, \theta_{1 : k_2-1})
= \sigma D_{x_{k_2}'(x,\theta_{1 : k_2-1})} m^{\frac{q_{k_2-1}}{2}} A_{k_2-1} \cdots m^{\frac{q_{k_1}}{2}} A_{k_1} D_{x_{k_1}'(x,\theta_{1 : k_1-1})} \in \R^{m_{k_2-1} \times m_{k_1-1}}.
\]
\end{definition}
The $k_1=k_2=k$ case is $B_{k, k}(x, \theta_{1 : k-1}) = \sigma D_{x_k'(x,\theta_{1 : k-1})} \in \R^{m_{k-1} \times m_{k-1}}$.

We denote the Frobenius inner products of the backpropagation matrices as
\[
X'_{k_1,k_2}(x_1,x_2,\theta_{1:k_2-1}) = \tr(B_{k_1, k_2}(x_1, \theta_{1 : k_2-1}) B_{k_1, k_2}(x_2, \theta_{1 : k_2-1})^*)
\]
for $x_1,x_2 \in \R^{m_0}$, $k_1 \leq k_2 \in [2:l]$ and $\theta_{1 : k_2-1} \in \Theta_{1 : k_2-1}$. Note that on the diagonal, we have the Frobenius norms $X'_{k_1,k_2}(x,x,\theta_{1:k_2-1}) = \Vert B_{k_1, k_2}(x, \theta_{1 : k_2-1}) \Vert_F^2$.

\begin{proposition}[Formula for $K_\theta(x_1,x_2)$]\label{prop:ntk_kernel_expression}~\\
Given $x_1, x_2 \in \R^{m_0}$, the entry $K_\theta(x_1,x_2) \in \R^{m_l \times m_l}$ equals
\begin{multline}\label{eq:ntk_kernel_value}
\sigma^{-2} \sum_{k=1}^{l-1} m^{q_k} X_k(x_1,x_2, \theta_{1:k-1}) A_l B_{k+1,l}(x_1, \theta_{1 : l-1}) {B_{k+1,l}(x_2, \theta_{1 : l-1})}^* A_l^* \\
+ X_l(x_1,x_2,\theta_{1:l-1}) \Id_{m_l}.
\end{multline}
\end{proposition}
\begin{proof}
Equation \eqref{eq:ntk_kernel_value} follows by applying recursively the formula \eqref{eq:jacobian} and noting that $M_{x_k(x_1, \theta_{1:k-1}),m_k} {M_{x_k(x_2, \theta_{1:k-1})},m_k}^* = \langle x_k(x_1, \theta_{1:k-1}), x_k(x_2, \theta_{1:k-1}) \rangle \Id_{m_k}$.
\end{proof}

\begin{proposition}[Expectation of $K_\theta(x_1,x_2)$]\label{prop:readout_expectation}~\\
Given $x_1, x_2 \in \R^{m_0}$ and $\theta_{1:l-1} \in \Theta_{1:l-1}$, we have that $\E_{A_l} K_\theta(x_1,x_2) \in \R^{m_l \times m_l}$ equals
\[
\left( \sum_{k=1}^{l-1} m^{q_k-q_l} X_k(x_1,x_2,\theta_{1:k-1}) X'_{k+1,l}(x_1,x_2,\theta_{1 : l-1}) + X_l(x_1,x_2,\theta_{1:l-1}) \right) \Id_{m_l}.
\]
\end{proposition}
\begin{proof}
For any $j_1,j_2 \in [1:m_l]$, $(K_\theta(x_1,x_2) - X_l(x_1,x_2,\theta_{1:l-1}) \Id_{m_l})_{j_1,j_2}$ can be written as a sum of inner products via \eqref{eq:ntk_kernel_value} as
\[
\sigma^{-2} \sum_{k=1}^{l-1} m^{q_k} X_k(x_1,x_2,\theta_{1:k-1}) \left\langle {A_l}_{j_1}, B_{k+1,l}(x_1, \theta_{1 : l-1}) {B_{k+1,l}(x_{i_2}, \theta_{1 : l-1})}^* {A_l}_{j_2} \right\rangle.
\]
If $j_1 \neq j_2$, since ${A_l}_{j_1}$ and ${A_l}_{j_2}$ are independent, the expectation of each term above is $0$. Otherwise, if $j_1 = j_2 = j$, by the trace trick we have that
\begin{multline*}
\E_{{A_l}_j} \sigma^{-2} m^{q_k} X_k(x_1,x_2,\theta_{1:k-1}) \left\langle {A_l}_j, B_{k+1,l}(x_1, \theta_{1 : l-1}) {B_{k+1,l}(x_2, \theta_{1 : l-1})}^* {A_l}_j \right\rangle \\
= m^{-q_l} m^{q_k} X_k(x_1,x_2,\theta_{1:k-1}) \tr\left( B_{k+1,l}(x_1, \theta_{1 : l-1}) {B_{k+1,l}(x_2, \theta_{1 : l-1})}^* \right) \\
= m^{q_k-q_l} X_k(x_1,x_2,\theta_{1:k-1}) X'_{k+1,l}(x_1,x_2,\theta_{1 : l-1}),
\end{multline*}
giving the claim.
\end{proof}

\begin{remark}[Optimal $q_1,\cdots,q_l$]\label{rem:optimal_q}~\\
In order for the terms in the above expectation not to blow up or vanish, we need $m^{q_k-q_l} = 1$ to hold for all $k \in [1:l-1]$. This is achieved precisely by letting $q \in \R$ and setting
\begin{equation}\label{eq:optimal_qs}
q_k = q \text{ for } k \in [1:l].
\end{equation}
This setting interpolates between the kernel regime ($q=0$) and the rich regime ($q=1$). Letting $q=0$ leads to the NTP of \citet{Jacotetal2018}, but $q=1$ gives the $\mu$P of \citet{Yangetal2022} only if $\gamma_1=\cdots=\gamma_{l-1}=1$. Using this scheme with $q=1$ can be seen as a principled extension of $\mu$P to the case of varying hidden layer sizes.
\end{remark}

\begin{proposition}[Concentration of $K_\theta(x_1,x_2)$] \label{prop:readout_concentration}~\\
Given $x_1, x_2 \in \R^{m_0}$ and $\theta_{1:l-1} \in \Theta_{1:l-1}$, for all $t \geq 0$ we have
\[
\Prob_l\left( \Vert K_\theta(x_1,x_2) - \E_{A_l} K_\theta(x_1,x_2) \Vert \geq t \right)
\leq 2e^{-\frac{t^2}{O\left( \Vert J(x_1,x_2,\theta_{1 : l-1}) \Vert_F \sqrt{m_l} \right)^2 + O\left( \Vert J(x_1,x_2,\theta_{1 : l-1}) \Vert m_l \right) t}}
\]
with $J(x_1,x_2,\theta_{1 : l-1}) \in \R^{m_{l-1} \times m_{l-1}}$ defined as
\[
J(x_1,x_2,\theta_{1 : l-1})
= \sum_{k=1}^{l-1} m^{q_k-q_l} X_k(x_1,x_2,\theta_{1:k-1}) B_{k+1, l}(x_1, \theta_{1 : l-1}) B_{k+1, l}(x_2, \theta_{1 : l-1})^*.
\]
\end{proposition}
\begin{proof}
Define $\hat{K}_\theta(x_1,x_2) = K_\theta(x_1,x_2) - X_l(x_1,x_2,\theta_{1:l-1}) \Id_{m_l}$, which is the NTK without the term corresponding to the last layer (which does not depend on $A_l$ by \eqref{eq:ntk_kernel_value}). Note that we have $\Vert K_\theta(x_1,x_2) - \E_{A_l} K_\theta(x_1,x_2) \Vert = \Vert \hat{K}_\theta(x_1,x_2) - \E_{A_l} \hat{K}_\theta(x_1,x_2) \Vert$, so it suffices to bound the latter. By \citet[Corollary~4.2.13]{Vershynin2018}, there exists a $\frac{1}{4}$-net $\hat{N} \subset \R^{m_l}$ of the unit sphere $\{ y \in \R^{m_l} : \Vert y \Vert = 1 \} \subset \R^{m_l}$ with cardinality $\vert \hat{N} \vert \leq 9^{m_l}$. By \citet[Exercise~4.4.3(b)]{Vershynin2018}, we have
\[
\left\Vert \hat{K}_\theta(x_1,x_2) - \E_{A_l} \hat{K}_\theta(x_1,x_2) \right\Vert 
\leq 2 \max_{y \in \hat{N}}\left\{ \left\vert \left\langle y, \left( \hat{K}_\theta(x_1,x_2) - \E_{A_l} \hat{K}_\theta(x_1,x_2) \right) y \right\rangle \right\vert \right\}.
\]

Denoting $J = J(x_1,x_2,\theta_{1 : l-1})$ for brevity, note that Proposition~\ref{prop:ntk_kernel_expression} implies $\hat{K}_\theta(x_1,x_2) = (\sigma^{-1} m^{\frac{q_l}{2}} A_l) J (\sigma^{-1} m^{\frac{q_l}{2}} A_l)^*$. Now fix $y \in \hat{N}$ and define $\hat{A} \in \R^{m_l m_{l-1}}$ as $\hat{A}_{(j_1-1) m_{l-1} + j_2} = \sigma^{-1} m^{\frac{q_l}{2}} {A_l}_{j_1,j_2}$ for $j_1 \in [1:m_l]$ and $j_2 \in [1:m_{l-1}]$ (i.e., $\hat{A}$ is $\sigma^{-1} m^{\frac{q_l}{2}} A_l$ flattened). Then we have $(\sigma^{-1} m^{\frac{q_l}{2}} A_l)^* y = (y \boxtimes \Id_{m_{l-1}})^* \hat{A}$, so that $\langle y, \hat{K}_\theta(x_1,x_2) y \rangle$ equals
\[
\langle y , (\sigma^{-1} m^{\frac{q_l}{2}} A_l) J (\sigma^{-1} m^{\frac{q_l}{2}} A_l)^*y \rangle
= \langle \hat{A}, (y \boxtimes \Id_{m_{l-1}} ) J (y \boxtimes \Id_{m_{l-1}} )^* \hat{A} \rangle.
\]

Having $\langle y, \hat{K}_\theta(x_1,x_2) y \rangle$ in this form lets us bound $\vert \langle y, (\hat{K}_\theta(x_1,x_2) - \E_{A_J} \hat{K}_\theta(x_1,x_2)) y \rangle \vert$ via the Hanson-Wright inequality \citep[Theorem~6.2.1]{Vershynin2018}. In order to do that, we need to bound the sub-gaussian norm of the coordinates of $\hat{A}$, as well as the operator and the Frobenius norms of the matrix $(y \boxtimes \Id_{m_{l-1}}) J (y \boxtimes \Id_{m_{l-1}})^*$. The random vector $\hat{A}$ has i.i.d. $\mathcal{N}(0,1)$ coordinates, so that it is coordinate-wise $O(1)$-sub-gaussian by \citep[Example~2.5.8 (i)]{Vershynin2018}. Since the operator norm is submultiplicative with respect to both the matrix product and the Kronecker product, we have
\[
\Vert (y \boxtimes \Id_{m_{l-1}}) J (y \boxtimes \Id_{m_{l-1}})^* \Vert
\leq \Vert y \Vert^2 \Vert \Id_{m_{l-1}} \Vert^2 \Vert J \Vert = \Vert J \Vert.
\]
By \citet[Exercise~6.3.3]{Vershynin2018}, we have the bound
\[
\Vert (y \boxtimes \Id_{m_{l-1}}) J (y \boxtimes \Id_{m_{l-1}})^* \Vert_F^2
\leq \Vert J \Vert_F^2.
\]

Applying \citet[Theorem~6.2.1]{Vershynin2018}, we have for all $t \geq 0$ the bound
\[
\Prob_l(\vert \langle y, (\hat{K}_\theta(x_1,x_2) - \E_{A_J} \hat{K}_\theta(x_1,x_2)) y \rangle \vert \geq t)
\leq 2e^{-\min\{ \frac{t^2}{O(\Vert J \Vert_F^2)}, \frac{t}{O(\Vert J \Vert)} \}} 
\leq 2e^{-\frac{t^2}{O(\Vert J \Vert_F^2) + O(\Vert J \Vert) t}}. 
\]
Unfixing $y \in \hat{N}$, by \citet[Lemma~2.2.13]{Vandervaartetal2023} and the bound $\log(1+9^{m_l}) \leq O(m_l)$ we have that
\[
\Prob_l\left( \Vert K_\theta(x_1,x_2) - \E_{A_l} K_\theta(x_1,x_2) \Vert \geq O(\Vert J \Vert_F \sqrt{m_l}) \sqrt{t} + O(\Vert J \Vert m_l) t \right)
\leq 2e^{-t}.
\]
This implies (see the paragraph below \citet[Example~2.2.12]{Vandervaartetal2023}) the conclusion.
\end{proof}

In order to apply the above concentration result, we need to bound the operator norms of the backpropagation matrices. For convenience, denote
\[
\kappa_\phi = \frac{\Vert \phi \Vert_L}{\Vert \phi \Vert_{\mathcal{N}(0,1)}} = (\vert a \vert + \vert b \vert) \sigma = \frac{\vert a \vert + \vert b \vert}{\sqrt{a^2 + b^2}} \in \left[ 1, \sqrt{2} \right].
\]

\begin{proposition}[Backpropagation matrices are bounded]\label{prop:backprop_norm_concentration}~\\
Given $x \in \R^{m_0}$ and $k_1 < k_2 \in [2:l]$, for all $t \geq 0$ we have that
\[
\Prob_{1:l-1}\left(
\gamma_{k_2-1}^{\frac{1}{2}} \left( \frac{\Vert B_{k_1, k_2}(x, \theta_{1 : k_2-1}) \Vert - O\left( \kappa_\phi^2 m_{k_2-1}^{-\frac{1}{2}} \Vert B_{k_1, k_2-1}(x, \theta_{1 : k_2-2}) \Vert_F \right)}{\Vert B_{k_1, k_2-1}(x, \theta_{1 : k_2-2}) \Vert} - 1 \right)_+ \geq t
\right)
\]
is at most $2e^{-\frac{t^2}{O\left( \kappa_\phi^2 m^{-\frac{1}{2}} \right)^2}}$.
\end{proposition}
\begin{proof}
First, consider $\theta_{1:k_2-2} \in \Theta_{1:k_2-2}$ and $\theta_{k_2:l-1} \in \Theta_{k_2:l-1}$ fixed and $A_{k_2-1} \in \Theta_{k_2-1}$ random. Denoting the preactivations $z_j = m^{\frac{q_{k_2-1}}{2}} \langle {A_{k_2-1}}_j, x_{k_2-1}(x, \theta_{1 : k_2-2}) \rangle$ for $j \in [1:m_{k_2-1}]$, the rows of $\sqrt{m_{k_2-1}} D_{x_{k_2}'(x,\theta_{1 : k_2-1})} m^{\frac{q_{k_2-1}}{2}} A_{k_2-1} \in \R^{m_{k_2-1} \times m_{k_2-2}}$ can be written as $\phi'(z_j) m^{\frac{q_{k_2-1}}{2}} {A_{k_2-1}}_j \in \R^{m_{k_2-2}}$. Note now that we have $\E_{A_{k_2-1}} (\phi'(z_j) m^{\frac{q_{k_2-1}}{2}} {A_{k_2-1}}_j)^{\otimes 2} = \E_{[u,v] \sim \mathcal{N}(0,\Sigma)} \phi'(u)^2 v^{\otimes 2}$ with
\[
\Sigma = \sigma^2 \left[ \begin{smallmatrix} \tau^2 & \tau \hat{x} \\ \tau \hat{x}^* & \Id_{m_{k_2-2}} \end{smallmatrix} \right] \in \R^{(1+m_{k_2-2}) \times (1+m_{k_2-2})},
\]
where we denoted $\tau = \Vert x_{k_2-1}(x, \theta_{1 : k_2-2}) \Vert$ and $\hat{x} = \tau^{-1} x_{k_2-1}(x, \theta_{1 : k_2-2})$. Taking the conditional of $v$ given $u$, the above expectation equals $\E_{u \sim \mathcal{N}(0,\sigma^2 \tau^2)} \phi'(u)^2 \E_{v \sim \mathcal{N}(\mu_{v \vert u},\Sigma_{v \vert u})} v^{\otimes 2} = \E_{u \sim \mathcal{N}(0,\sigma^2 \tau^2)} \phi'(u)^2 (\mu_{v \vert u}^{\otimes 2} + \Sigma_{v \vert u})$ with $\mu_{v \vert u} = u \tau^{-1} \hat{x}$ and $\Sigma_{v \vert u} = \sigma^2 (\Id_{m_{k_2-2}} - \hat{x}^{\otimes 2})$. As $\E_{u \sim \mathcal{N}(0,\sigma^2 \tau^2)} \phi'(u)^2 u^2 \tau^{-2} = \E_{u \sim \mathcal{N}(0,\sigma^2 \tau^2)} \phi'(u)^2 \sigma^2 = 1$, we then have 
\[
\E_{A_{k_2-1}} (\phi'(z_j) m^{\frac{q_{k_2-1}}{2}} {A_{k_2-1}}_j)^{\otimes 2} = \Id_{m_{k_2-2}},
\]
i.e., the i.i.d. random vectors $\phi'(z_j) m^{\frac{q_{k_2-1}}{2}} {A_{k_2-1}}_j$ are isotropic. Clearly we also have the bound $\Vert \phi'(z_j) m^{\frac{q_{k_2-1}}{2}} {A_{k_2-1}}_j \Vert_{\psi_2} \leq O(\kappa_\phi)$. 

Denoting $B = B_{k_1, k_2-1}(x, \theta_{1 : k_2-2})$, by \citet[Exercise~9.1.8]{Vershynin2018} we get the bound
\[
\Prob_{k_2-1}\left(
\Vert B_{k_1, k_2}(x, \theta_{1 : k_2-1}) \Vert \geq (1+t) \Vert B \Vert + O(\kappa_\phi^2 m_{k_2-1}^{-\frac{1}{2}} w(B))
\right)
\leq 2e^{-\frac{t^2}{O(\kappa_\phi^2 m_{k_2-1}^{-\frac{1}{2}})^2}},
\]
where $w(B) = \E_{g \sim \mathcal{N}(0,\Id_{m_{k_2-2}})} \sup_{y \in \R^{m_{k_1-1}} : \Vert y \Vert = 1}\{ \langle B y, g \rangle \}$ is the Gaussian width of the image of the unit sphere $\{ y \in \R^{m_{k_1-1}} : \Vert y \Vert = 1 \}$ under $B$. Noting that
\[
w(B) = \E_{g \sim \mathcal{N}\left( 0, \Id_{m_{k_2-2}} \right)} \sup_{y \in \R^{m_{k_1-1}} : \Vert y \Vert = 1}\{ \langle y, B^* g \rangle \} = \E_{g \sim \mathcal{N}\left( 0, \Id_{m_{k_2-2}} \right)} \Vert B^* g \Vert,
\]
we have $w(B) \leq \sqrt{\E_{g \sim \mathcal{N}(0,\Id_{m_{k_2-2}})} \Vert B^* g \Vert^2} = \Vert B \Vert_F$ by Jensen's inequality and \citet[Exercise~6.3.1]{Vershynin2018}. Substituting into the concentration bound above, we have
\[
\Prob_{k_2-1}\left(
\gamma_{k_2-1}^{\frac{1}{2}} \left( \frac{\Vert B_{k_1, k_2}(x, \theta_{1 : k_2-1}) \Vert - O(\kappa^2 m_{k_2-1}^{-\frac{1}{2}} \Vert B \Vert_F)}{\Vert B \Vert} - 1 \right)_+ \geq t
\right)
\leq 2e^{-\frac{t^2}{O(\kappa^2 m^{-\frac{1}{2}})^2}}.
\]

In other words, with the event $E_t \in \mathcal{B}(\Theta_{1:l-1})$ defined as
\[
E_t = \left\{ \theta_{1:l-1} \in \Theta_{1:l-1} : \gamma_{k_2-1}^{\frac{1}{2}} \left( \frac{\Vert B_{k_1, k_2}(x, \theta_{1 : k_2-1}) \Vert - O(\kappa^2 m_{k_2-1}^{-\frac{1}{2}} \Vert B \Vert_F)}{\Vert B \Vert} - 1 \right)_+ \geq t \right\}
\]
and $\chi_{E_t} : \Theta_{1:l-1} \to \{0,1\}$ being the indicator function of $E_t$ we have that
\[
\int_{\Theta_{k_2-1}} \chi_{E_t}(\theta_{1:k_2-2},A_{k_2-1},\theta_{k_2:l-1}) d\Prob_{k_2-1}(A_{k_2-1}) \leq 2e^{-\frac{t^2}{O(\kappa^2 m^{-\frac{1}{2}})^2}}
\]
for all $\theta_{1:k_2-2} = [A_1,\cdots,A_{k_2-2}] \in \Theta_{1:k_2-2}$ and $\theta_{k_2:l-1} = [A_{k_2},\cdots,A_{l-1}] \in \Theta_{k_2:l-1}$. Denoting $\Theta_{1:l-1 \setminus k_2-1} = \Theta_{1:k_2-2} \times \Theta_{k_2:l-1}$ and $\Prob_{1:l-1 \setminus k_2-1} = \Prob_1 \otimes \cdots \otimes \Prob_{k_2-2} \otimes \Prob_{k_2} \otimes \cdots \otimes \Prob_{l-1}$, the Fubini-Tonelli theorem then implies that $\int_{\Theta_{1:l-1}} \chi_{E_t}(\theta_{1:l-1}) d\Prob(\theta_{1:l-1})$ equals
\begin{multline*}
\int_{\Theta_{1:l-1 \setminus k_2-1}} \left( \int_{\Theta_{k_2-1}} \chi_{E_t}(\theta_{1:k_2-2},A_{k_2-1},\theta_{k_2:l-1}) d\Prob_{k_2-1}(A_{k_2-1}) \right) d\Prob_{1:l-1 \setminus k_2-1}(\theta_{1:k_2-2}, \theta_{k_2:l-1}) \\
\leq \int_{\Theta_{1:l-1 \setminus k_2-1}} 2e^{-\frac{t^2}{O(\kappa^2 m^{-\frac{1}{2}})^2}} d\Prob_{1:l-1 \setminus k_2-1}(\theta_{1:k_2-2}, \theta_{k_2:l-1})
= 2e^{-\frac{t^2}{O(\kappa^2 m^{-\frac{1}{2}})^2}}.
\end{multline*}
Hence $\Prob_{k_2-1}$ can be replaced by $\Prob_{1:l-1}$ in the above concentration bound, giving the claim.
\end{proof}

Define the cosine map $\varrho : [-1,1] \to [-1,1]$ for all $\rho \in [-1,1]$ as
\[
\varrho(\rho) 
= \sigma^2 \int \phi(u_1) \phi(u_2) d\mathcal{N}\left( [u_1,u_2] \left\vert 0,\left[ \begin{smallmatrix} 1 & \rho \\ \rho & 1 \end{smallmatrix} \right] \right. \right),
\]
which is the dual function of $\phi$ in the sense of \citet{Danielyetal2016} at the EOC. It is responsible for the propagation of the cosines of the activations in the infinitely wide limit (see \citet[Proposition~9]{mlpsateoc1}).

\begin{proposition}[Expectation of $X_k(x_1,x_2,\theta_{1:k-1})$]\label{prop:fwd_inner_product_expectation}~\\
Given $x_1,x_2 \in \R^{m_0}$, $k \in [2:l]$ and $\theta_{1:k-2} \in \Theta_{1:k-2}$, we have
\[
\E_{A_{k-1}} X_k(x_1,x_2,\theta_{1:k-1})
= \tau_{k-1}(x_1, \theta_{1:k-2})\tau_{k-1}(x_2, \theta_{1:k-2}) \varrho(\rho_{k-1}(x_1,x_2,\theta_{1:k-2})).
\]
\end{proposition}
\begin{proof}
Denote the preactivations $z_{i,j} = m^{\frac{q_{k-1}}{2}} \langle {A_{k-1}}_j, x_{k-1}(x_i, \theta_{1 : k-2}) \rangle$ for $i \in [1:2]$ and $j \in [1:m_{k-1}]$. We then have
\[
\E_{A_{k-1}} X_k(x_1,x_2,\theta_{1:k-1})
= \frac{1}{m_{k-1}} \sum_{j=1}^{m_{k-1}}\E_{{A_{k-1}}_j} \phi(z_{1,j}) \phi(z_{2,j}).
\]
As all the rows of $A_{j-1}$ are i.i.d., all these expectations are equal and for any fixed $j \in [1:m_{k-1}]$ the above equals $\E_{{A_{k-1}}_j} \phi( z_{i_1,j} ) \phi( z_{i_2,j} )$. Since $m^{\frac{q_{k-1}}{2}}{A_{k-1}}_j \sim \mathcal{N}( 0, \sigma^2 \Id_{m_{k-2}})$, the expectation $\E_{A_{k-1}} X_k(x_1,x_2,\theta_{1:k-1})$ equals
\begin{multline*}
\E_{v \sim \mathcal{N}\left( 0, \sigma^2 \Id_{m_{k-2}} \right)} \phi\left( \left\langle v, x_{k-1}(x_1, \theta_{1 : k-2}) \right\rangle \right) \phi\left( \left\langle v, x_{k-1}(x_2, \theta_{1 : k-2}) \right\rangle \right) \\
= \int \phi(u_1) \phi(u_2) d\mathcal{N}\left([u_1,u_2] \left| 0, \sigma^2 \left[ \begin{smallmatrix} \tau_1^2 & \tau_1 \tau_2 \rho_{k-1}(x_1,x_2,\theta_{1:k-2}) \\ \tau_1 \tau_2 \rho_{k-1}(x_1,x_2,\theta_{1:k-2}) & \tau_2^2 \end{smallmatrix} \right] \right. \right) \\
= \tau_1 \tau_2 \varrho(\rho_{k-1}(x_1,x_2,\theta_{1:k-2}))
\end{multline*}
using the homogeneity of $\phi$, where we denoted $\tau_i = \tau_{k-1}(x_i, \theta_{1:k-2})$ for $i \in [1:2]$.
\end{proof}

We could study the concentration of the activation inner products directly, but it would lead to suboptimal bounds. Factoring out the norms gives the cosines, to which we can associate the corresponding cosine distances. \citet[Proposition~13]{mlpsateoc1} tells us that these quantities scale as $O(k^{-1})$ across depth. We will study the concentration of proxies to the cosine distances, which we will later relate to the actual cosine distances via the law of cosines. Define the squared cosine distance map $\zeta : [0,1] \to [0,1]$ as $\zeta(z) = \frac{1-\varrho(1 - 2z)}{2}$ for $z \in [0,1]$ (see \citet[Proposition~11]{mlpsateoc1} for its properties).

\begin{proposition}[Concentration of cosine distances of activations]\label{prop:cosine_distance_concentration}~\\
Given $x_1, x_2 \in \R^{m_0}$ and $k \in [2:l]$, for all $t \geq 0$ we have that
\[
\Prob_{1:l-1}\left( \zeta(z)^{-\frac{1}{2}} \frac{\zeta(z)}{z} \gamma_{k-1}^{\frac{1}{2}} \left\vert \left\Vert \frac{1}{2} \frac{x_k(x_1, \theta_{1 : k-1})}{\Vert x_{k-1}(x_1, \theta_{1 : k-2}) \Vert} - \frac{1}{2} \frac{x_k(x_2, \theta_{1 : k-1})}{\Vert x_{k-1}(x_2, \theta_{1 : k-2}) \Vert} \right\Vert - \zeta(z)^{\frac{1}{2}} \right\vert \geq t \right)
\]
is at most $2e^{-\frac{t^2}{O\left( \kappa_\phi^2 m^{-\frac{1}{2}} \right)^2}}$ with
\[
z = \frac{1 - \rho_{k-1}(x_1,x_2,\theta_{1:k-2})}{2} = \left\Vert \frac{1}{2} \frac{x_{k-1}(x_1, \theta_{1 : k-2})}{\Vert x_{k-1}(x_1, \theta_{1 : k-2}) \Vert} - \frac{1}{2} \frac{x_{k-1}(x_2, \theta_{1 : k-2})}{\Vert x_{k-1}(x_2, \theta_{1 : k-2}) \Vert} \right\Vert^2 \in [0,1].
\]
\end{proposition}
\begin{proof}
First, consider $\theta_{1:k-2} \in \Theta_{1:k-2}$ and $\theta_{k:l-1} \in \Theta_{j:l-1}$ fixed and $A_{k-1} \in \Theta_{k-1}$ random. Denote the normalized preactivations $z_{i,j} = m^{\frac{q_{k-1}}{2}} \langle {A_{k-1}}_j , \frac{ x_{k-1}(x_i, \theta_{1 : k-2}) }{ \Vert x_{k-1}(x_i, \theta_{1 : k-2}) \Vert } \rangle$ for $i \in [1:2]$ and $j \in [1:m_{k-1}]$, so that $z_{i,j} \sim \mathcal{N}(0,\sigma^2)$. Note that $\frac{x_k(x_i, \theta_{1 : k-1})}{\Vert x_{k-1}(x_i, \theta_{1 : k-2}) \Vert} = [ \frac{1}{\sqrt{m_{k-1}}} \phi(z_{i,j}) : j \in [1:m_{k-1}] ]$ by the homogeneity of $\phi$. Consider the decomposition $\phi(z_{1,j}) - \phi(z_{2,j}) = (a z_{1,j} + b \vert z_{1,j} \vert) - (a z_{2,j} + b \vert z_{2,j} \vert) = a (z_{1,j} - z_{2,j}) + b (\vert z_{1,j} \vert - \vert z_{2,j} \vert)$. We have $[z_{1,j},z_{2,j}] \sim \mathcal{N}(0, \sigma^2 \left[ \begin{smallmatrix} 1 & \rho \\ \rho & 1 \end{smallmatrix} \right])$ with $\rho = \rho_{k-1}(x_1,x_2,\theta_{1:k-2}) = 1-2z$, so that $\frac{1}{2} z_{1,j} - \frac{1}{2} z_{2,j} \sim \mathcal{N}(0,\sigma^2 z)$ and therefore $\Vert \frac{1}{2} z_{1,j} - \frac{1}{2} z_{2,j} \Vert_{\psi_2} \leq O(\sigma \sqrt{z})$ by \citet[Example~2.5.8(i)]{Vershynin2018}. On the other hand, by the reverse triangle inequality we have $\vert \frac{1}{2} \vert z_{1,j} \vert - \frac{1}{2} \vert z_{2,j} \vert \vert \leq \vert \frac{1}{2} z_{1,j} - \frac{1}{2} z_{2,j} \vert$, so that $\Vert \frac{1}{2} \vert z_{1,j} \vert - \frac{1}{2} \vert z_{2,j} \vert \Vert_{\psi_2} \leq O(\sigma \sqrt{z})$ as well. Hence by subadditivity we get the bound $\Vert \frac{1}{2} \phi(z_{1,j}) - \frac{1}{2} \phi(z_{2,j}) \Vert_{\psi_2} = O((\vert a \vert + \vert b \vert) \sigma \sqrt{z}) \leq O(\kappa_\phi \sqrt{z})$. Squaring and centering, by \citet[Lemma~2.7.6]{Vershynin2018} and \citet[Exercise~2.7.10]{Vershynin2018} we get $\Vert (\frac{1}{2} \phi(z_{1,j}) - \frac{1}{2} \phi(z_{2,j}))^2 - \E_{{A_{k-1}}_j} (\frac{1}{2} \phi(z_{1,j}) - \frac{1}{2} \phi(z_{2,j}))^2 \Vert_{\psi_1} = O(\kappa_\phi^2 z)$. We can compute that the expectation $\E_{{A_{k-1}}_j} (\frac{1}{2} \phi(z_{1,j}) - \frac{1}{2} \phi(z_{2,j}))^2$ equals
\[
\frac{1}{4} \E_{{A_{k-1}}_j} \phi(z_{1,j})^2 + \frac{1}{4} \E_{{A_{k-1}}_j} \phi(z_{2,j})^2 - \frac{1}{2} \E_{{A_{k-1}}_j} \phi(z_{1,j})\phi(z_{2,j}) = \frac{1 - \varrho(\rho)}{2} = \zeta(z).
\]
Since
\[
\left\Vert \frac{1}{2} \frac{x_k(x_1, \theta_{1 : k-1})}{\Vert x_{k-1}(x_1, \theta_{1 : k-2}) \Vert} - \frac{1}{2} \frac{x_k(x_2, \theta_{1 : k-1})}{\Vert x_{k-1}(x_2, \theta_{1 : k-2}) \Vert} \right\Vert^2 
= \frac{1}{m_{k-1}} \sum_{j=1}^{m_{k-1}} (\phi(z_{1,j}) - \phi(z_{2,j}))^2,
\]
we can apply \citet[Corollary~2.8.3]{Vershynin2018} to get that for any $\delta \geq 0$,
\begin{multline*}
\Prob_{k-1}\left( \left\vert \left\Vert \frac{1}{2} \frac{x_k(x_1, \theta_{1 : k-1})}{\Vert x_{k-1}(x_1, \theta_{1 : k-2}) \Vert} - \frac{1}{2} \frac{x_k(x_2, \theta_{1 : k-1})}{\Vert x_{k-1}(x_2, \theta_{1 : k-2}) \Vert} \right\Vert^2 - \zeta(z) \right\vert 
\geq \max\{ \delta, \delta^2 \} \zeta(z) \right) \\
\leq 2e^{-\min\left\{ \left( \frac{\max\{ \delta, \delta^2 \} \zeta(z)}{O(\kappa_\phi^2 z)} \right)^2, \frac{\max\{ \delta, \delta^2 \} \zeta(z)}{O(\kappa_\phi^2 z)} \right\} m_{k-1}}
\leq 2e^{-\frac{\delta^2}{O(\kappa_\phi^2 \frac{z}{\zeta(z)})^2 m_{k-1}^{-1}}}.
\end{multline*}
By the implication $\vert c_1 - c_2 \vert \geq \delta c_2 \implies \vert c_1^2 - c_2^2 \vert \geq \max\{ \delta, \delta^2 \} c_2^2$ that holds for all $c_1, c_2, \delta \geq 0$, we then have that
\[
\Prob_{k-1}\left( \left\vert \left\Vert \frac{1}{2} \frac{x_k(x_1, \theta_{1 : k-1})}{\Vert x_{k-1}(x_1, \theta_{1 : k-2}) \Vert} - \frac{1}{2} \frac{x_k(x_2, \theta_{1 : k-1})}{\Vert x_{k-1}(x_2, \theta_{1 : k-2}) \Vert} \right\Vert - \sqrt{\zeta(z)} \right\vert \geq \delta \sqrt{\zeta(z)} \right)
\]
is at most $2e^{-\frac{\delta^2}{O(\kappa_\phi^2 \frac{z}{\zeta(z)})^2 m_{k-1}^{-1}}}$. Letting $t = \frac{\delta}{\frac{z}{\zeta(z)} \gamma_{k-1}^{-\frac{1}{2}}}$, we get that
\[
\Prob_{k-1}\left( \zeta(z)^{-\frac{1}{2}} \frac{\zeta(z)}{z} \gamma_{k-1}^{\frac{1}{2}} \left\vert \left\Vert \frac{1}{2} \frac{x_k(x_1, \theta_{1 : k-1})}{\Vert x_{k-1}(x_1, \theta_{1 : k-2}) \Vert} - \frac{1}{2} \frac{x_k(x_2, \theta_{1 : k-1})}{\Vert x_{k-1}(x_2, \theta_{1 : k-2}) \Vert} \right\Vert - \sqrt{\zeta(z)} \right\vert \geq t \right)
\]
is at most $2e^{-\frac{t^2}{O(\kappa_\phi^2)^2 m^{-1}}}$. As this holds for all $\theta_{1:k-2} \in \Theta_{1:k-2}$ and $\theta_{k:l-1} \in \Theta_{j:l-1}$, by the Fubini-Tonelli theorem the above bound still holds with $\Prob_{k-1}$ replaced by $\Prob_{1:l-1}$.
\end{proof}

Denote $\Delta_\phi = \frac{b^2}{a^2+b^2}$, which determines the rate at which inverse cosine distances increase in the infinitely wide limit by \citet[Proposition~13]{mlpsateoc1}.

\begin{remark}[Optimal $\gamma_1,\cdots,\gamma_{l-1}$]\label{rem:optimal_gammas}~\\
Based on \citet[Proposition~13]{mlpsateoc1}, we expect for all $k \in [1:l-1]$ that $\zeta(z)^{-\frac{1}{2}} \frac{\zeta(z)}{z} \approx \Delta_\phi \frac{4}{3\pi} (k-1)$ with sufficient concentration, where we denoted $z = \frac{1 - \rho_{k-1}(x_1,x_2,\theta_{1:k-2})}{2}$.  Proposition~\ref{prop:cosine_distance_concentration} suggests setting
\begin{equation}\label{eq:optimal_gammas}
\gamma_k = k^2 \text{ for all } k \in [1:l-1],
\end{equation}
so that $\zeta(z)^{-\frac{1}{2}} \frac{\zeta(z)}{z} \gamma_{k-1}^{\frac{1}{2}} \approx \Delta_\phi \frac{4}{3\pi} (k-1)^2$ and the concentration error of the (proxies of the) cosine distances will scale as $O(k^{-2})$. It will turn out that this is necessary and sufficient for the inverse cosine distances to increase linearly, as they do in the infinitely wide limit. Figure~\ref{plot:icd_concentration} demonstrates empirically that with this setting, the errors of inverse cosine distances are of the same order in each layer, while the error grows linearly for $\gamma_k=k$ and quadratically for $\gamma_k=1$.
\end{remark}

\begin{figure}[t]
%\vskip 0.2in
\begin{center}
\centerline{
\begin{tikzpicture}
\begin{groupplot}[group style={group size= 1 by 3}, height=5cm, width=0.9\linewidth]
%\pgfplotsset{cycle list/Dark2}
\nextgroupplot[ylabel={$m_k=m$}, xmin=0, xmax = 33, error bars/y dir=both, error bars/y explicit, legend pos=north west, xlabel = $k$]
%\begin{axis}[
%    xlabel = $k$,
%    ylabel = $\left( \frac{1 - \rho_k(x_1,x_2,\theta_{1:k-1})}{2} \right)^{-\frac{1}{2}}$,
%	error bars/y dir=both,
%	error bars/y explicit,
%    legend columns=2,
%    legend style={draw=none},
%    xmin=0,
%    xmax=33,
%	x post scale=1.9
%]
\addlegendentry{$m=16$}
\addplot[mycolor2] table [x=Step, y=Value, y error=Std, col sep=comma] {csvs/ntk/run-inverse_cosine_distances_param_nu_q_1_r_0_MNIST_n_2_m_16_l_32_a_0.0_b_1.0_samples_1000_2025-01-24_17_36_20_365186-tag-w_error_mean.csv};
\addlegendentry{$m=32$}
\addplot[mycolor5] table [x=Step, y=Value, y error=Std, col sep=comma] {csvs/ntk/run-inverse_cosine_distances_param_nu_q_1_r_0_MNIST_n_2_m_32_l_32_a_0.0_b_1.0_samples_1000_2025-01-24_17_36_28_656021-tag-w_error_mean.csv};
\addlegendentry{$m=64$}
\addplot[mycolor8] table [x=Step, y=Value, y error=Std, col sep=comma] {csvs/ntk/run-inverse_cosine_distances_param_nu_q_1_r_0_MNIST_n_2_m_64_l_32_a_0.0_b_1.0_samples_1000_2025-01-24_17_36_30_283151-tag-w_error_mean.csv};
\coordinate (top) at (rel axis cs:0,1);
\nextgroupplot[ylabel={$m_k=km$}, xmin=0, xmax = 33, error bars/y dir=both, error bars/y explicit, legend pos=north west, xlabel = $k$]
\addlegendentry{$m=8$}
\addplot[mycolor2] table [x=Step, y=Value, y error=Std, col sep=comma] {csvs/ntk/run-inverse_cosine_distances_param_nu_q_1_r_1_MNIST_n_2_m_8_l_32_a_0.0_b_1.0_samples_1000_2025-01-24_17_35_12_530305-tag-w_error_mean.csv};
\addlegendentry{$m=16$}
\addplot[mycolor5] table [x=Step, y=Value, y error=Std, col sep=comma] {csvs/ntk/run-inverse_cosine_distances_param_nu_q_1_r_1_MNIST_n_2_m_16_l_32_a_0.0_b_1.0_samples_1000_2025-01-24_17_35_16_983650-tag-w_error_mean.csv};
\addlegendentry{$m=32$}
\addplot[mycolor8] table [x=Step, y=Value, y error=Std, col sep=comma] {csvs/ntk/run-inverse_cosine_distances_param_nu_q_1_r_1_MNIST_n_2_m_32_l_32_a_0.0_b_1.0_samples_1000_2025-01-24_17_36_15_564766-tag-w_error_mean.csv};
\nextgroupplot[ylabel={$m_k=k^2m$}, xmin=0, xmax = 33, error bars/y dir=both, error bars/y explicit, legend pos=north west, xlabel = $k$]
\addlegendentry{$m=4$}
\addplot[mycolor2] table [x=Step, y=Value, y error=Std, col sep=comma] {csvs/ntk/run-inverse_cosine_distances_param_nu_q_1_r_2_MNIST_n_2_m_4_l_32_a_0.0_b_1.0_samples_1000_2025-01-24_17_35_04_081855-tag-w_error_mean.csv};
\addlegendentry{$m=8$}
\addplot[mycolor5] table [x=Step, y=Value, y error=Std, col sep=comma] {csvs/ntk/run-inverse_cosine_distances_param_nu_q_1_r_2_MNIST_n_2_m_8_l_32_a_0.0_b_1.0_samples_1000_2025-01-24_17_35_07_774580-tag-w_error_mean.csv};
\addlegendentry{$m=16$}
\addplot[mycolor8] table [x=Step, y=Value, y error=Std, col sep=comma] {csvs/ntk/run-inverse_cosine_distances_param_nu_q_1_r_2_MNIST_n_2_m_16_l_32_a_0.0_b_1.0_samples_1000_2025-01-24_17_35_09_724525-tag-w_error_mean.csv};
\coordinate (bot) at (rel axis cs:1,0);
\end{groupplot}
%\end{axis}
\path (top-|current bounding box.west) -- node[anchor=south,rotate=90] {$\left\vert \left( \frac{1 - \rho_k(x_1,x_2,\theta_{1:k-1})}{2} \right)^{-\frac{1}{2}} - \left( \frac{1 - \varrho^{\circ (k-1)}(\rho_1(x_1,x_2))}{2} \right)^{-\frac{1}{2}} \right\vert$} (bot-|current bounding box.west);
\end{tikzpicture}
}
\caption{Error between the empirical and limiting inverse cosine distances for different layer width patterns. Depicted are the means and standard deviations of the errors across depth in $32$-layer MLPs with $(a,b)=(0,1)$ taken from $1000$ random pairs $x_1,x_2$ drawn from MNIST, each with a new initial parameter $\theta$.}
\label{plot:icd_concentration}
\end{center}
%\vskip -0.2in
\end{figure}
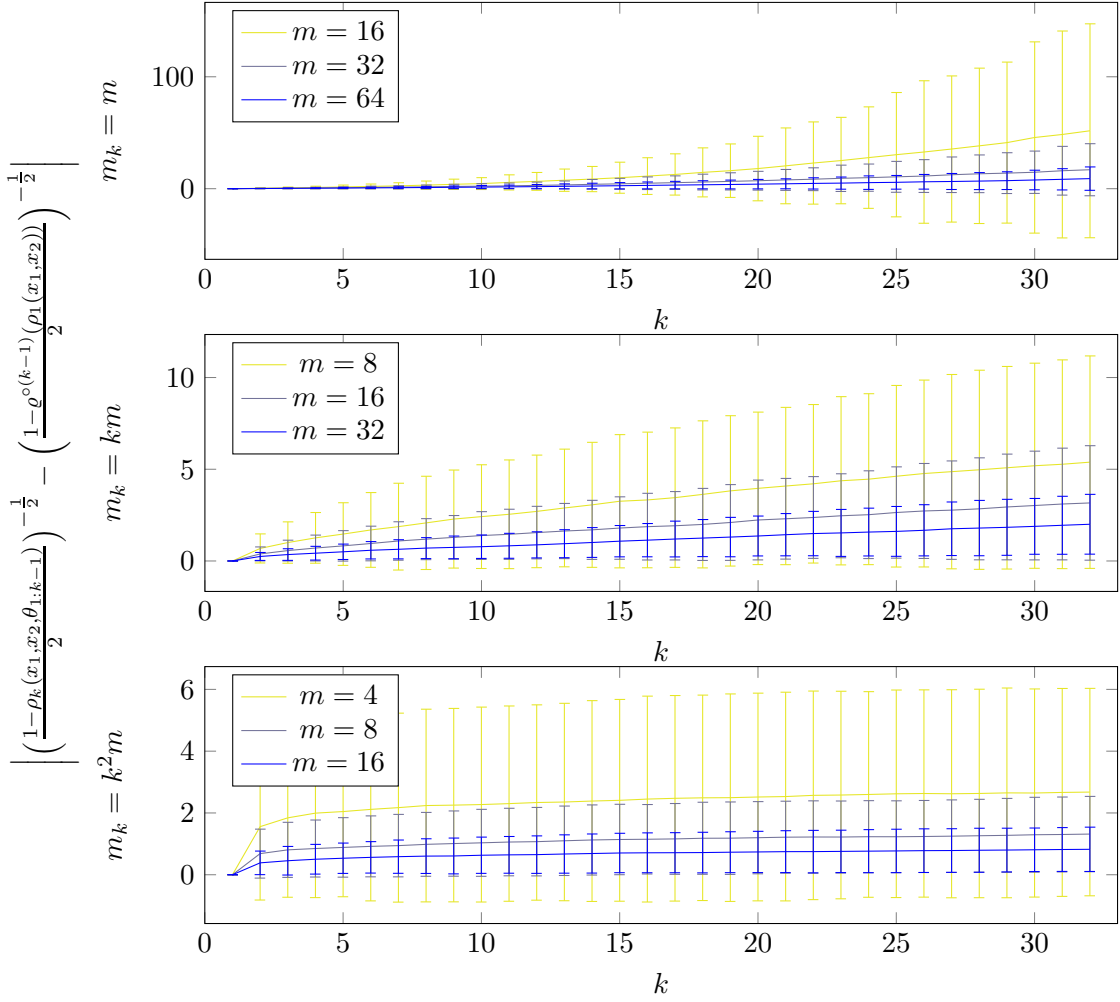

\begin{proposition}[Concentration of norms of activations]\label{prop:norm_concentration}~\\
Given $x \in \R^{m_0}$ and $k \in [2:l]$, for all $t \geq 0$ we have
\[
\Prob_{1:l-1}\left( \gamma_{k-1}^{\frac{1}{2}} \left\vert \frac{\Vert x_k(x, \theta_{1 : k-1}) \Vert}{\Vert x_{k-1}(x, \theta_{1 : k-2}) \Vert} - 1 \right\vert \geq t \right)
\leq 2e^{-\frac{t^2}{O\left( \kappa_\phi^2 m^{-\frac{1}{2}} \right)^2}}.
\]
\end{proposition}
\begin{proof}
Note that replacing $\frac{1}{2} \frac{x_k(x_1, \theta_{1 : k-1})}{\Vert x_{k-1}(x_1, \theta_{1 : k-2}) \Vert}$ and $\frac{1}{2} \frac{x_k(x_2, \theta_{1 : k-1})}{\Vert x_{k-1}(x_2, \theta_{1 : k-2}) \Vert}$ by $\frac{x_k(x, \theta_{1 : k-1})}{\Vert x_{k-1}(x, \theta_{1 : k-2}) \Vert}$ and $0$ in the proof of Proposition~\ref{prop:cosine_distance_concentration} gives the claim.
\end{proof}

By \citet[Proposition~7]{mlpsateoc1} and \citet[Proposition~9]{mlpsateoc1}, we have for all $\rho \in [-1,1]$ that
\[
\varrho'(\rho) 
= \sigma^2 \int \phi'(u_1) \phi'(u_2) d\mathcal{N}\left( [u_1,u_2] \left\vert 0,\left[ \begin{smallmatrix} 1 & \rho \\ \rho & 1 \end{smallmatrix} \right] \right. \right),
\]
i.e., taking the dual commutes with differentiation as shown in \citet{Danielyetal2016}. Additional justification for the notation $X'$ is the fact that the Frobenius inner products of the backpropagation matrices concentrate around the images of the cosines under $\varrho'$.

\begin{proposition}[Expectation of $X'_{k,k}(x_1,x_2,\theta_{1:k-1})$]~\\
Given $x_1,x_2 \in \R^{m_0}$, $k \in [2:l]$ and $\theta_{1:k-2} \in \Theta_{1:k-2}$, we have
\[
\E_{A_{k-1}} X'_{k,k}(x_1,x_2,\theta_{1:k-1})
= \varrho'(\rho_{k-1}(x_1,x_2,\theta_{1:k-2})).
\]
\end{proposition}
\begin{proof}
Denoting the normalized preactivations $z_{i,j} = m^{\frac{q_{k-1}}{2}} \langle {A_{k-1}}_j, \frac{x_{k-1}(x_i, \theta_{1 : k-2})}{\Vert x_{k-1}(x_i, \theta_{1 : k-2}) \Vert} \rangle$ for $i \in [1:2]$ and $j \in [1:m_{k-1}]$, we have that $X'_{k,k}(x_1,x_2,\theta_{1:k-1})$ equals
\begin{multline*}
\tr\left( \sigma^2 D_{x_k'(x_1,\theta_{1 : k-1})} D_{x_k'(x_2,\theta_{1 : k-1})}^* \right)
= \sigma^2 \langle x_k'(x_1,\theta_{1 : k-1}), x_k'(x_2,\theta_{1 : k-1}) \rangle \\
= \sigma^2 \frac{1}{m_{k-1}} \sum_{j=1}^{m_{k-1}} \phi'(\Vert x_{k-1}(x_1, \theta_{1 : k-2}) \Vert z_{1,j}) \phi'(\Vert x_{k-1}(x_2, \theta_{1 : k-2}) \Vert z_{2,j}),
\end{multline*}
which further equals $\sigma^2 \frac{1}{m_{k-1}} \sum_{j=1}^{m_{k-1}} \phi'(z_{1,j}) \phi'(z_{2,j})$ as $\phi'(ts)=\phi'(s)$ for all $t > 0$. Since $\phi'(z_{1,j}) \phi'(z_{2,j})$ are i.i.d. for all $j \in [1:m_{k-1}]$, we have
\begin{multline*}
\E_{A_{k-1}} X'_{k,k}(x_1,x_2,\theta_{1:k-1})
= \sigma^2 \E_{{A_{k-1}}_j} \phi'( z_{1,j} ) \phi'( z_{2,j} ) \\
= \sigma^2 \E_{v \sim \mathcal{N}( 0, \sigma^2 \Id_{m_{k-2}} )} \phi'\left( \left\langle v, \frac{x_{k-1}(x_1, \theta_{1 : k-2})}{\Vert x_{k-1}(x_1, \theta_{1 : k-2}) \Vert} \right\rangle \right) \phi'\left( \left\langle v, \frac{x_{k-1}(x_2, \theta_{1 : k-2})}{\Vert x_{k-1}(x_2, \theta_{1 : k-2}) \Vert} \right\rangle \right) \\
= \sigma^2 \int \phi'(u_1) \phi'(u_2) d\mathcal{N}\left([u_1,u_2] \left| 0, \sigma^2 \left[ \begin{smallmatrix} 1 & \rho_{k-1}(x_1,x_2,\theta_{1:k-2}) \\ \rho_{k-1}(x_1,x_2,\theta_{1:k-2}) & 1 \end{smallmatrix} \right] \right. \right) 
= \varrho'(\rho),
\end{multline*}
giving the claim.
\end{proof}

\begin{proposition}[Expectation of $X'_{k_1,k_2}(x,x,\theta_{1:k_2-1})$]\label{prop:bwd_expectation_diagonal}~\\
Given $x \in \R^{m_0}$, $k_1 < k_2 \in [2:l]$ and $\theta_{1:k_2-1} \in \Theta_{1:k_2-1}$ such that $\Vert x_{k_2-1}(x, \theta_{1 : k_2-2}) \Vert > 0$, we have
\[
\E_{A_{k_2-1}} X'_{k_1,k_2}(x,x,\theta_{1:k_2-1})
= X'_{k_1,k_2-1}(x,x,\theta_{1:k_2-2}).
\]
\end{proposition}
\begin{proof}
Denoting the preactivations $z_j = m^{\frac{q_{k_2-1}}{2}} \langle {A_{k_2-1}}_j, x_{k_2-1}(x, \theta_{1 : k_2-2}) \rangle$ for $j \in [1:m_{k_2-1}]$, we have that $X'_{k_1,k_2}(x,x,\theta_{1:k_2-1})$ equals
\[
\Vert B_{k_1, k_2}(x, \theta_{1 : k_2-1}) \Vert_F^2
= \frac{1}{m_{k_2-1}} \sum_{j=1}^{m_{k_2-1}} 
\phi'(z_j)^2 \left\Vert B_{k_1,k_2-1}(x,\theta_{1:k_2-2})^* m^{\frac{q_{k_2-1}}{2}} {A_{k_2-1}}_j^* \right\Vert^2.
\]
As the terms in the sum above are i.i.d. for $j \in [1:m_{k_2-1}]$, we have
\[
\E_{A_{k_2-1}} X'_{k_1,k_2}(x,x,\theta_{1:k_2-1})
= \E_{{A_{k_2-1}}_j} \phi'( z_j )^2
\left\Vert B_{k_1,k_2-1}(x,\theta_{1:k_2-2})^* m^{\frac{q_{k_2-1}}{2}} {A_{k_2-1}}_j^* \right\Vert^2.
\]
Since $m^{\frac{q_{k_2-1}}{2}}{A_{k_2-1}}_j \sim \mathcal{N}( 0, \sigma^2 \Id_{m_{k_2-2}})$, the above equals
\[
\E_{v \sim \mathcal{N}( 0, \sigma^2 \Id_{m_{k_2-2}} )} \phi'( \langle v, x_{k_2-1}(x, \theta_{1 : k_2-2}) \rangle )^2 \left\Vert B_{k_1,k_2-1}(x,\theta_{1:k_2-2})^* v \right\Vert^2.
\]
Note that we can write the above as $\E_{(u,v) \sim \mathcal{N}(0,\Sigma)} \phi'( u )^2 \Vert v \Vert^2$ with $\Sigma = \left[ \begin{smallmatrix} \Sigma_u & \Sigma_{uv}^* \\ \Sigma_{uv} & \Sigma_v \end{smallmatrix} \right] \in \mathbb{S}^{1+m_{k_1-1}}_+$, where $\Sigma_u = \sigma^2 \Vert x_{k_2-1}(x, \theta_{1 : k_2-2}) \Vert^2 > 0$, $\Sigma_v = \sigma^2 B_{k_1,k_2-1}(x,\theta_{1:k_2-2})^* B_{k_1,k_2-1}(x,\theta_{1:k_2-2}) \in \mathbb{S}^{m_{k_1-1}}_+$ and $\Sigma_{uv} = \sigma^2 B_{k_1,k_2-1}(x,\theta_{1:k_2-2})^* x_{k_2-1}(x, \theta_{1 : k_2-2}) \in \R^{m_{k_1-1}}$. The conditional distribution of $v$ given $u$ is a normal distribution with mean $\mu_{v \vert u} = u \Sigma_u^{-1} \Sigma_{uv}$ and covariance $\Sigma_{v \vert u} = \Sigma_v - \Sigma_u^{-1} \Sigma_{uv}^{\otimes 2}$. Since $\E_{v \sim \mathcal{N}(\mu_{v \vert u}, \Sigma_{v \vert u})} \Vert v \Vert^2 = \Vert \mu_{v \vert u} \Vert^2 + \tr(\Sigma_{v \vert u}) = \Sigma_u^{-2} u^2 \Vert \Sigma_{uv} \Vert^2 + \tr(\Sigma_v) - \Sigma_u^{-1} \Vert \Sigma_{uv} \Vert^2$, we get that 
\begin{multline*}
\E_{(u,v) \sim \mathcal{N}(0,\Sigma)} \phi'(u)^2 \Vert v \Vert^2 
= \E_{u \sim \mathcal{N}(0,\Sigma_u)} \phi'(u)^2 \E_{v \sim \mathcal{N}(\mu_{v \vert u},\Sigma_{v \vert u})} \Vert v \Vert^2 \\
= \E_{u \sim \mathcal{N}(0,\Sigma_u)} \phi'(u)^2 (\tr(\Sigma_v) + (\Sigma_u^{-2} u^2 - \Sigma_u^{-1}) \Vert \Sigma_{uv} \Vert^2).
\end{multline*}
Denoting $x = x_{k_2-1}(x, \theta_{1 : k_2-2})$ and $B = B_{k_1,k_2-1}(x,\theta_{1:k_2-2})$, the above equals
\[
\E_{u \sim \mathcal{N}(0,\sigma^2 \Vert x \Vert^2)} \phi'( u )^2
\left( \sigma^2 \tr(B^* B) + \sigma^2 \left( \left( \sigma^{-1} \Vert x \Vert^{-1} u \right)^2 - 1 \right) \left\Vert B^* \frac{x}{\Vert x \Vert} \right\Vert^2 \right).
\]
As $\sigma^2 \E_{u \sim \mathcal{N}(0,1)} \phi'(u)^2 = \sigma^2 \E_{u \sim \mathcal{N}(0,1)} (\phi'(u)u)^2 = 1$, we get that
\[
\E_{(u,v) \sim \mathcal{N}(0,\Sigma)} \phi'(u)^2 \Vert v \Vert^2 
= \tr(B^* B) = \tr(B B^*) = X'_{k_1,k_2-1}(x,x,\theta_{1:k_2-2})
\]
giving the claim.
\end{proof}

Computing the expectation in the offdiagonal case has been done using the GIA since \citet{Jacotetal2018}, which has been shown to be true asymptotically in the infinitely wide limit by \citet{Yang2020}. Instead of relying on this assumption, we quantify the finite depth gradient dependence error in the expectation in terms of the activation cosines and the norms of the backpropagation matrices. The result below shows that MLPs with $(a,b)$-ReLUs at the EOC enjoy approximate gradient independence at finite width.

\begin{proposition}[Expectation of $X'_{k_1,k_2}(x_1,x_2,\theta_{1:k_2-1})$]\label{prop:bwd_expectation_offdiagonal}~\\
Given $x_1,x_2 \in \R^{m_0}$, $k_1 < k_2 \in [2:l]$ and $\theta_{1:k_2-2} \in \Theta_{1:k_2-2}$ such that $\rho_{k_2-1}(x_1,x_2,\theta_{1:k_2-2}) \in (-1,1)$, we have
\begin{multline*}
\left\vert \E_{A_{k_2-1}} X'_{k_1,k_2}(x_1,x_2,\theta_{1:k_2-1}) - \varrho'(\rho_{k_2-1}(x_1,x_2,\theta_{1:k_2-2})) X'_{k_1,k_2-1}(x_1,x_2,\theta_{1:k_2-2}) \right\vert \\
\leq \Delta_\phi \frac{8}{\pi} \sqrt{\frac{1-\rho_{k_2-1}(x_1,x_2,\theta_{1:k_2-2})}{1+\rho_{k_2-1}(x_1,x_2,\theta_{1:k_2-2})}} \Vert B_{k_1,k_2-1}(x_1,\theta_{1:k_2-2}) \Vert \Vert B_{k_1,k_2-1}(x_2,\theta_{1:k_2-2}) \Vert.
\end{multline*}
\end{proposition}
\begin{proof}
Denoting the preactivations $z_{i,j} = m^{\frac{q_{k_2-1}}{2}} \langle {A_{k_2-1}}_j, x_{k_2-1}(x_i, \theta_{1 : k_2-2}) \rangle$ for $i \in [1:2]$ and $j \in [1:m_{k_2-1}]$, we have that $X'_{k_1,k_2}(x_1,x_2,\theta_{1:k_2-1})$ equals
\begin{multline*}
\tr\left( D_{x_{k_2}'(x_1,\theta_{1 : k_2-1})} m^{\frac{q_{k_2-1}}{2}} A_{k_2-1} B_{k_1,k_2-1}(x_1,\theta_{1:k_2-2})
\right. \\ \left. 
B_{k_1,k_2-1}(x_2,\theta_{1:k_2-2})^* m^{\frac{q_{k_2-1}}{2}} A_{k_2-1}^* D_{x_{k_2}'(x_2,\theta_{1 : k_2-1})} \right) \\
= \frac{1}{m_{k_2-1}} \sum_{j=1}^{m_{k_2-1}} \phi'(z_{1,j}) \phi'(z_{2,j}) 
\left\langle B_{k_1,k_2-1}(x_1,\theta_{1:k_2-2})^* m^{\frac{q_{k_2-1}}{2}} {A_{k_2-1}}_j^*,
\right. \\ \left.
B_{k_1,k_2-1}(x_2,\theta_{1:k_2-2})^* m^{\frac{q_{k_2-1}}{2}} {A_{k_2-1}}_j^* \right\rangle.
\end{multline*}
As the terms in the sum above are i.i.d. for $j \in [1:m_{k_2-1}]$, we have
\begin{multline*}
\E_{A_{k_2-1}} X'_{k_1,k_2}(x_1,x_2,\theta_{1:k_2-1}) \\
= \E_{{A_{k_2-1}}_j} \phi'(z_{1,j}) \phi'(z_{2,j}) \left\langle B_{k_1,k_2-1}(x_1,\theta_{1:k_2-2})^* m^{\frac{q_{k_2-1}}{2}} {A_{k_2-1}}_j^*,
\right. \\ \left.
B_{k_1,k_2-1}(x_2,\theta_{1:k_2-2})^* m^{\frac{q_{k_2-1}}{2}} {A_{k_2-1}}_j^* \right\rangle.
\end{multline*}
Since $m^{\frac{q_{k_2-1}}{2}}{A_{k_2-1}}_j \sim \mathcal{N}( 0, \sigma^2 \Id_{m_{k_2-2}})$, we can write the above as
\begin{multline*}
\E_{v \sim \mathcal{N}( 0, \sigma^2 \Id_{m_{k_2-2}} )} \phi'( \langle v, x_{k_2-1}(x_1, \theta_{1 : k_2-2}) \rangle ) \phi'( \langle v, x_{k_2-1}(x_2, \theta_{1 : k_2-2}) \rangle ) \\
\left\langle B_{k_1,k_2-1}(x_1,\theta_{1:k_2-2})^* v,
B_{k_1,k_2-1}(x_2,\theta_{1:k_2-2})^* v \right\rangle.
\end{multline*}
Note that this equals $\E_{[u_1,u_2,v_1,v_2] \sim \mathcal{N}(0,\Sigma)} \phi'(u_1) \phi'(u_2) \langle v_1, v_2 \rangle$ where $u=[u_1,u_2]\in \R^2$, $v=[v_1,v_2] \in \R^{2m_{k_1-1}}$ and the covariance matrix is defined blockwise as $\Sigma = \left[ \begin{smallmatrix} \Sigma_u & \Sigma_{uv}^* \\ \Sigma_{uv} & \Sigma_v \end{smallmatrix} \right] \in \mathbb{S}^{2+2m_{k_1-1}}_+$. The $u$-covariance is 
\begin{equation}\label{eq:mean-Y-sigma-u}
\Sigma_u = \sigma^2 \left[ \begin{smallmatrix} \Vert x_{k_2-1}(x_1, \theta_{1 : k_2-2}) \Vert^2 & \langle x_{k_2-1}(x_1, \theta_{1 : k_2-2}), x_{k_2-1}(x_2, \theta_{1 : k_2-2}) \rangle \\ \langle x_{k_2-1}(x_1, \theta_{1 : k_2-2}), x_{k_2-1}(x_2, \theta_{1 : k_2-2}) \rangle & \Vert x_{k_2-1}(x_2, \theta_{1 : k_2-2}) \Vert^2 \end{smallmatrix} \right] \in \mathbb{S}^2_+,
\end{equation}
the $v$-covariance $\Sigma_v \in \mathbb{S}^{2m_{k_1-1}}_+$ is
\begin{equation}\label{eq:mean-Y-sigma-v}
\sigma^2 \left[ \begin{smallmatrix} B_{k_1,k_2-1}(x_1,\theta_{1:k_2-2})^* B_{k_1,k_2-1}(x_1,\theta_{1:k_2-2}) & B_{k_1,k_2-1}(x_2,\theta_{1:k_2-2})^* B_{k_1,k_2-1}(x_1,\theta_{1:k_2-2}) \\ B_{k_1,k_2-1}(x_1,\theta_{1:k_2-2})^* B_{k_1,k_2-1}(x_2,\theta_{1:k_2-2}) & B_{k_1,k_2-1}(x_2,\theta_{1:k_2-2})^* B_{k_1,k_2-1}(x_2,\theta_{1:k_2-2}) \end{smallmatrix} \right]
\end{equation}
and the cross-covariance $\Sigma_{uv} \in \R^{2m_{k_1-1} \times 2 }$ is
\begin{equation}\label{eq:mean-Y-sigma-uv}
\sigma^2 \left[ \begin{smallmatrix} B_{k_1,k_2-1}(x_1,\theta_{1:k_2-2})^* x_{k_2-1}(x_1, \theta_{1 : k_2-2}) & B_{k_1,k_2-1}(x_2,\theta_{1:k_2-2})^* x_{k_2-1}(x_1, \theta_{1 : k_2-2}) \\ B_{k_1,k_2-1}(x_1,\theta_{1:k_2-2})^* x_{k_2-1}(x_2, \theta_{1 : k_2-2}) & B_{k_1,k_2-1}(x_2,\theta_{1:k_2-2})^* x_{k_2-1}(x_2, \theta_{1 : k_2-2}) \end{smallmatrix} \right].
\end{equation}

Note that our assumption $\rho_{k_2-1}(x_1,x_2,\theta_{1:k_2-2}) \in (-1,1)$ implies that $\Sigma_u$ is invertible. In particular, letting $\tau_1 = \tau_{k_2-1}(x_1, \theta_{1 : k_2-2})$, $\tau_2 = \tau_{k_2-1}(x_2, \theta_{1 : k_2-2})$ and $\rho = \rho_{k_2-1}(x_1,x_2,\theta_{1:k_2-2})$ (so that $\Sigma_u = \sigma^2 \left[ \begin{smallmatrix} \tau_1^2 & \tau_1 \tau_2 \rho \\ \tau_1 \tau_2 \rho & \tau_2^2 \end{smallmatrix} \right]$), we have $\Sigma_u^{-1} = \sigma^{-2} (1-\rho^2)^{-1} \left[ \begin{smallmatrix} \tau_1^{-2} & -\tau_1^{-1} \tau_2^{-1} \rho \\ -\tau_1^{-1} \tau_2^{-1} \rho & \tau_2^{-2} \end{smallmatrix} \right]$. The conditional distribution of $[v_1,v_2]$ given $[u_1,u_2]$ is a normal distribution with mean $\mu_{v \vert u}=\Sigma_{uv} \Sigma_u^{-1} [u_1, u_2]$ and covariance $\Sigma_{v \vert u} = \Sigma_v - \Sigma_{uv} \Sigma_u^{-1} \Sigma_{uv}^*$. Thus we have
\begin{multline*}
\E_{[u_1,u_2,v_1,v_2] \sim \mathcal{N}( 0,\Sigma)} \phi'(u_1) \phi'(u_2) \langle v_1, v_2 \rangle \\
= \E_{[u_1,u_2] \in \mathcal{N}(0,\Sigma_u)}
\phi'(u_1) \phi'(u_2) \E_{[v_1,v_2] \sim \mathcal{N}(\mu_{v \vert u},\Sigma_{v \vert u})}
\langle v_1, v_2 \rangle,
\end{multline*}
where the inner expectation can be computed as
\begin{multline*}
\E_{[v_1,v_2] \sim \mathcal{N}(\mu_{v \vert u}, \Sigma_{v \vert u})}
\langle v_1, v_2 \rangle
 = \E_{[v_1,v_2] \sim \mathcal{N}(\mu_{v \vert u}, \Sigma_{v \vert u})}
\sum_{j=1}^{m_{k_1-1}} {v_1}_j {v_2}_j \\
= \sum_{j=1}^{m_{k_1-1}} \E_{[v_1,v_2] \sim \mathcal{N}(\mu_{v \vert u}, \Sigma_{v \vert u})} {v_1}_j {v_2}_j
= \sum_{j=1}^{m_{k_1-1}} {{\mu_{v \vert u}}_1}_j {{\mu_{v \vert u}}_2}_j + {{\Sigma_{v \vert u}}_{1,2}}_{j,j} \\
= \tr({\Sigma_v}_{1,2}) + (\langle {\mu_{v \vert u}}_1, {\mu_{v \vert u}}_2 \rangle - \tr({\Sigma_{uv} \Sigma_u^{-1} \Sigma_{uv}^*}_{1,2})).
\end{multline*}
The first term gives $\E_{[u_1,u_2] \sim \mathcal{N}( 0,\Sigma_u)} \phi'(u_1)\phi'(u_2) \tr({\Sigma_v}_{1,2})$, which equals
\[
\varrho'(\rho) \tr(B_{k_1,k_2-1}(x_1,\theta_{1:k_2-2}) B_{k_1,k_2-1}(x_2,\theta_{1:k_2-2})^*) = \varrho'(\rho) X'_{k_1,k_2-1}(x_1,x_2,\theta_{1:k_2-2}).
\]

The other two terms result in the gradient dependence error
\[
\epsilon = \E_{[u_1,u_2] \sim \mathcal{N}( 0,\Sigma_u)} \phi'(u_1)\phi'(u_2) \left( \langle {\mu_{v \vert u}}_1, {\mu_{v \vert u}}_2 \rangle - \tr({\Sigma_{uv} \Sigma_u^{-1} \Sigma_{uv}^*}_{1,2})\right).
\]
For brevity, denote $B_i = B_{k_1,k_2-1}(x_i,\theta_{1:k_2-2})$ and $\hat{x}_i = \frac{x_{k_2-1}(x_i, \theta_{1 : k_2-2})}{\Vert x_{k_2-1}(x_i, \theta_{1 : k_2-2}) \Vert}$ for $i \in [1:2]$, $w_1 = \sigma^{-1} (1-\rho^2)^{-\frac{1}{2}} (\tau_1^{-1} u_1 - \rho \tau_2^{-1} u_2)$ and $w_2 = \sigma^{-1} (1-\rho^2)^{-\frac{1}{2}} (\tau_2^{-1} u_2 - \rho \tau_1^{-1} u_1)$. Note that we have $\Sigma_u^{-1} [u_1,u_2] = \sigma^{-1} (1-\rho^2)^{-\frac{1}{2}} [\tau_1^{-1} w_1 , \tau_2^{-1} w_2]$, so that ${\mu_{v \vert u}}_1 = \sigma (1-\rho^2)^{-\frac{1}{2}} (\tau_1^{-1} w_1 B_1^* x_1 + \tau_2^{-1} w_2 B_1^* x_2) = \sigma (1-\rho^2)^{-\frac{1}{2}} (w_1 B_1^* \hat{x}_1 + w_2 B_1^* \hat{x}_2)$ and similarly ${\mu_{v \vert u}}_2 
= \sigma (1-\rho^2)^{-\frac{1}{2}} (w_1 B_2^* \hat{x}_1 + w_2 B_2^* \hat{x}_2)$. Therefore $\langle {\mu_{v \vert u}}_1 , {\mu_{v \vert u}}_2 \rangle$ equals
\[
\sigma^2 (1-\rho^2)^{-1} \left( w_1^2 \langle B_1^* \hat{x}_1, B_2^* \hat{x}_1 \rangle
+ w_2^2 \langle B_1^* \hat{x}_2, B_2^* \hat{x}_2 \rangle
+ w_1 w_2 \langle B_1^* \hat{x}_1, B_2^* \hat{x}_2 \rangle
+ w_1 w_2 \langle B_1^* \hat{x}_2, B_2^* \hat{x}_1 \rangle \right).
\]
We also have
\[
\Sigma_{uv} \Sigma_u^{-1}
= (1-\rho^2)^{-1} \left[ \begin{smallmatrix} 
 \tau_1^{-1} B_1^* \hat{x}_1 - \tau_1^{-1} \rho B_1^* \hat{x}_2 & 
 -\tau_2^{-1} \rho B_1^* \hat{x}_1 + \tau_2^{-1} B_1^* \hat{x}_2 \\ 
 \tau_1^{-1} B_2^* \hat{x}_1 - \tau_1^{-1} \rho B_2^* \hat{x}_2 & 
 -\tau_2^{-1} \rho B_2^* \hat{x}_1 + \tau_2^{-1} B_2^* \hat{x}_2
\end{smallmatrix} \right]
\]
so that ${\Sigma_{uv} \Sigma_u^{-1} \Sigma_{uv}^*}_{1,2}$ equals
\[
\sigma^2 (1-\rho^2)^{-1} \left( 
B_2^* \hat{x}_1 \otimes B_1^* \hat{x}_1
- \rho B_2^* \hat{x}_2 \otimes B_1^* \hat{x}_1
- \rho B_2^* \hat{x}_1 \otimes B_1^* \hat{x}_2
+ B_2^* \hat{x}_2 \otimes B_1^* \hat{x}_2
\right)
\]
and therefore $\tr({\Sigma_{uv} \Sigma_u^{-1} \Sigma_{uv}^*}_{1,2})$ equals
\[
\sigma^2 (1-\rho^2)^{-1} \left( 
\langle B_1^* \hat{x}_1, B_2^* \hat{x}_1 \rangle
+ \langle B_1^* \hat{x}_2, B_2^* \hat{x}_2 \rangle
- \rho \langle B_1^* \hat{x}_1, B_2^* \hat{x}_2 \rangle
- \rho \langle B_1^* \hat{x}_2, B_2^* \hat{x}_1 \rangle
\right).
\]
Hence we have that the gradient dependence term $\epsilon$ equals
\begin{multline*}
\sigma^2 (1-\rho^2)^{-1} \E_{[u_1,u_2] \sim \mathcal{N}(0,\Sigma_u)} \phi'(u_1) \phi'(u_2) 
\left( (w_1^2 - 1) \langle B_1^* \hat{x}_1, B_2^* \hat{x}_1 \rangle
+ (w_2^2 - 1) \langle B_1^* \hat{x}_2, B_2^* \hat{x}_2 \rangle 
\right. \\ \left. \vphantom{w_1^2}
+ (w_1 w_2 + \rho) \langle B_1^* \hat{x}_1, B_2^* \hat{x}_2 \rangle
+ (w_1 w_2 + \rho) \langle B_1^* \hat{x}_2, B_2^* \hat{x}_1 \rangle
\right).
\end{multline*}
Symbolic integration gives
\[
\E_{[u_1,u_2] \sim \mathcal{N}(0,\Sigma_u)} \phi'(u_1) \phi'(u_2) (w_1^2 - 1) = \E_{[u_1,u_2] \sim \mathcal{N}(0,\Sigma_u)} \phi'(u_1) \phi'(u_2) (w_2^2 - 1) = -\rho b^2 \frac{2}{\pi} \sqrt{1-\rho^2}
\]
and $\E_{[u_1,u_2] \sim \mathcal{N}(0,\Sigma_u)} \phi'(u_1) \phi'(u_2) (w_1 w_2 + \rho) = b^2 \frac{2}{\pi} \sqrt{1-\rho^2}$, so that
\[
\epsilon 
= c_b \frac{2}{\pi} (1-\rho^2)^{-\frac{1}{2}} \left( \langle B_1^* \hat{x}_1, B_2^* (\hat{x}_2 - \rho \hat{x}_1) \rangle + \langle B_1^* \hat{x}_2, B_2^* (\hat{x}_1 - \rho \hat{x}_2) \rangle \right).
\]
Noting that $\langle B_1^* \hat{x}_1, B_2^* (\hat{x}_2 - \rho \hat{x}_1) \rangle + \langle B_1^* \hat{x}_2, B_2^* (\hat{x}_1 - \rho \hat{x}_2) \rangle = \langle B_1^* (\hat{x}_1 - \hat{x}_2), B_2^* (\hat{x}_2 - \hat{x}_1) \rangle + (1-\rho) (\langle B_1^* \hat{x}_1, B_2^* \hat{x}_1 \rangle + \langle B_1^* \hat{x}_2, B_2^* \hat{x}_2 \rangle)$ and $\Vert \hat{x}_1 - \hat{x}_2 \Vert = \sqrt{2(1-\rho)}$, we get the bound
\[
\epsilon \leq c_b \frac{2}{\pi} (1-\rho^2)^{-\frac{1}{2}} \Vert B_1 B_2^* \Vert (\Vert \hat{x}_1 - \hat{x}_2 \Vert^2 + (1-\rho) (\Vert \hat{x}_1 \Vert^2 + \Vert \hat{x}_2 \Vert^2))
= c_b \frac{8}{\pi} \sqrt{\frac{1-\rho}{1+\rho}} \Vert B_1 B_2^* \Vert,
\]
giving the claim.
\end{proof}

\begin{proposition}[Concentration of $X'_{k,k}(x_1,x_2,\theta_{1:k-1})$]\label{prop:bwd_concentration_1}~\\
Given $x_1,x_2 \in \R^{m_0}$ and $k \in [2:l]$, for all $t \geq 0$ we have
\[
\Prob_{1:l-1}\left( \gamma_{k-1}^{\frac{1}{2}} \left\vert X'_{k,k}(x_1,x_2,\theta_{1:k-1}) - \varrho'(\rho_{k-1}(x_1,x_2,\theta_{1:k-2})) \right\vert \geq t \right) 
\leq 2e^{-\frac{t^2}{O\left( \kappa_\phi^2 m^{-\frac{1}{2}} \right)^2 + O\left( \kappa_\phi^2 m^{-1} \right) t}}.
\]
\end{proposition}
\begin{proof}
First, consider $\theta_{1:k-2} \in \Theta_{1:k-2}$ and $\theta_{k:l-1} \in \Theta_{k:l-1}$ fixed and $A_{k-1} \in \Theta_{k-1}$ random. Denoting the normalized preactivations $z_{i,j} = m^{\frac{q_{k-1}}{2}} \langle {A_{k-1}}_j, \frac{x_{k-1}(x_i, \theta_{1 : k-2})}{\Vert x_{k-1}(x_i, \theta_{1 : k-2}) \Vert} \rangle$ for $i \in [1:2]$ and $j \in [1:m_{k-1}]$, we have
\[
X'_{k,k}(x_1,x_2,\theta_{1:k-1})
= \sigma^2 \frac{1}{m_{k-1}} \sum_{j=1}^{m_{k-1}} \phi'(z_{1,j}) \phi'(z_{2,j}).
\]
As $\Vert \phi'(z_{i,j}) \Vert_{\psi_2} \leq O(\vert a \vert + \vert b \vert)$ for all $i \in [1:2]$ and $j \in [1:m_{k-1}]$ by \citet[Example~2.5.8(iii)]{Vershynin2018}, via \citet[Lemma~2.7.7]{Vershynin2018} and \citet[Exercise~2.7.10]{Vershynin2018} we get that $\Vert \sigma^2 \phi'(z_{1,j}) \phi'(z_{2,j}) - \sigma^2 \E_{{A_{k-1}}_j} \phi'(z_{1,j}) \phi'(z_{2,j}) \Vert_{\psi_1} \leq O(\kappa_\phi^2)$. The expectation equals $\sigma^2 \E_{{A_{k-1}}_j} \phi'(z_{1,j}) \phi'(z_{2,j}) = \varrho'(\rho_{k-1}(x_1,x_2,\theta_{1:k-2}))$. By \citet[Corollary~2.8.3]{Vershynin2018} we then have for all $t \geq 0$ that
\[
\Prob_{k-1}\left( \gamma_{k-1}^{\frac{1}{2}} \left\vert X'_{k,k}(x_1,x_2,\theta_{1:k-1}) - \varrho'(\rho) \right\vert \geq t \right)
\]
is at most
\[
2e^{-\min\left\{ \frac{t^2}{O(\gamma_{k-1}^{\frac{1}{2}} \kappa_\phi^2)^2}, \frac{t}{O(\gamma_{k-1}^{\frac{1}{2}} \kappa_\phi^2)} \right\} m_{k-1}} 
\leq 2e^{-\frac{t^2}{O(\kappa_\phi^2)^2 m^{-1} + O(\kappa_\phi^2) m^{-1} t}}.
\]
As this holds for all $\theta_{1:k-2} \in \Theta_{1:k-2}$ and $\theta_{k:l-1} \in \Theta_{k:l-1}$, by the Fubini-Tonelli theorem the above bound still holds with $\Prob_{k-1}$ replaced by $\Prob_{1:l-1}$.
\end{proof}

\begin{proposition}[Concentration of $X'_{k_1,k_2}(x_1,x_2,\theta_{1:k_2-1})$]\label{prop:bwd_concentration_2}~\\
Given $k_1 < k_2 \in [2:l]$ and $x_1,x_2 \in \R^{m_0}$, for all $t \geq 0$ we have that
\[
\Prob_{1:l-1}\left( \frac{\gamma_{k_2-1}^{\frac{1}{2}} \left\vert X'_{k_1,k_2}(x_1,x_2,\theta_{1:k_2-1}) - \E_{A_{k_2-1}} X'_{k_1,k_2}(x_1,x_2,\theta_{1:k_2-1}) \right\vert}{\sqrt{X'_{k_1,k_2-1}(x_1,x_1,\theta_{1:k_2-2})} \sqrt{X'_{k_1,k_2-1}(x_2,x_2,\theta_{1:k_2-2})}} \geq t \right)
\]
is at most $2e^{-\frac{t^2}{O\left( \kappa_\phi^2 m^{-\frac{1}{2}} \right)^2 + O\left( \kappa_\phi^2 m^{-1} \right) t}}$.
\end{proposition}
\begin{proof}
First, consider $\theta_{1:k_2-2} \in \Theta_{1:k_2-2}$ and $\theta_{k_2:l-1} \in \Theta_{k_2:l-1}$ fixed and $A_{k_2-1} \in \Theta_{k_2-1}$ random. Denoting the preactivations $z_{i,j} = m^{\frac{q_{k_2-1}}{2}} \langle {A_{k_2-1}}_j, x_{k_2-1}(x_i, \theta_{1 : k_2-2}) \rangle$ for $i \in [1:2]$ and $j \in [1:m_{k_2-1}]$, note that
\[
X'_{k_1,k_2}(x_1,x_2,\theta_{1:k_2-1})
= \frac{1}{m_{k_2-1}} \sum_{j=1}^{m_{k_2-1}} \phi'(z_{1,j}) \phi'(z_{2,j})
\left\langle B_1^* m^{\frac{q_{k_2-1}}{2}} {A_{k_2-1}}_j, B_2^* m^{\frac{q_{k_2-1}}{2}} {A_{k_2-1}}_j \right\rangle
\]
with $B_i = B_{k_1,k_2-1}(x_i,\theta_{1:k_2-2})$ for $i \in [1:2]$. The absolute value of each summand is bounded by 
\[
(\vert a \vert + \vert b \vert)^2 \left\Vert B_1^* m^{\frac{q_{k_2-1}}{2}} {A_{k_2-1}}_j \right\Vert \left\Vert B_2^* m^{\frac{q_{k_2-1}}{2}} {A_{k_2-1}}_j \right\Vert.
\]
The random vector $\sigma^{-1} m^{\frac{q_{k_2-1}}{2}} {A_{k_2-1}}_j$ has multivariate normal distribution with mean $0$ and covariance $\Id_{m_{k_2-2}}$. By \cite[Theorem~6.3.2]{Vershynin2018}, for $i \in [1:2]$ the random variable $\Vert B_i^* m^{\frac{q_{k_2-1}}{2}} {A_{k_2-1}}_j \Vert - \sigma \Vert B_i \Vert_F$ is $O(\sigma \Vert B_i \Vert)$-sub-gaussian. Thus, by the triangle inequality for the sub-gaussian norm, we have that $\Vert \Vert B_i^* m^{\frac{q_{k_2-1}}{2}} {A_{k_2-1}}_j \Vert \Vert_{\psi_2} \leq O( \sigma \Vert B_i \Vert_F )$, so that
\[
\left\Vert \phi'(z_{1,j}) \phi'(z_{2,j})
\left\langle B_1^* m^{\frac{q_{k_2-1}}{2}} {A_{k_2-1}}_j, B_2^* m^{\frac{q_{k_2-1}}{2}} {A_{k_2-1}}_j \right\rangle
- \E_{A_{k_2-1}} X'_{k_1,k_2}(x_1,x_2,\theta_{1:k_2-1}) \right\Vert_{\psi_1}
\]
is bounded by $O(\kappa^2 \Vert B_1 \Vert_F \Vert B_2 \Vert_F)$ via \citet[Lemma~2.7.7]{Vershynin2018} and \citet[Exercise~2.7.10]{Vershynin2018}. Scaling by $\gamma_{k_2-1}^{\frac{1}{2}} \Vert B_1 \Vert_F^{-1} \Vert B_2 \Vert_F^{-1}$, via \citet[Corollary~2.8.3]{Vershynin2018} we then have
\begin{multline*}
\Prob_{k_2-1}\left( \gamma_{k_2-1}^{\frac{1}{2}} \Vert B_1 \Vert_F^{-1} \Vert B_2 \Vert_F^{-1} \left\vert X'_{k_1,k_2}(x_1,x_2,\theta_{1:k_2-1}) 
- \E_{A_{k_2-1}} X'_{k_1,k_2}(x_1,x_2,\theta_{1:k_2-1}) \right\vert \geq t  \right) \\
\leq 2e^{-\min\left\{ \frac{t^2}{O(\gamma_{k_2-1}^{\frac{1}{2}} \kappa_\phi^2)^2}, \frac{t}{O(\gamma_{k_2-1}^{\frac{1}{2}} \kappa_\phi^2)} \right\} m_{k_2-1}} 
\leq 2e^{-\frac{t^2}{O(\kappa_\phi^2)^2 m^{-1} + O(\kappa_\phi^2) m^{-1} t}}
\end{multline*}
for all $t \geq 0$. As this holds for all $\theta_{1:k_2-2} \in \Theta_{1:k_2-2}$ and $\theta_{k_2:l-1} \in \Theta_{k_2:l-1}$, by the Fubini-Tonelli theorem the above bound still holds with $\Prob_{k_2-1}$ replaced by $\Prob_{1:l-1}$. We get the claim as $\Vert B_i \Vert_F = \sqrt{X'_{k_1,k_2-1}(x_i,x_i,\theta_{1:k_2-2})}$ for $i \in [1:2]$.
\end{proof}

\subsection{Limiting Concentration of the NTK}\label{limiting}

Building on the results of \S~\ref{layerwise} and \citet{mlpsateoc1}, we are going to prove that the NTK matrix concentrates around its infinitely wide limit, both defined below.

\begin{definition}[Neural Tangent Kernel matrix]\label{def:ntk_matrix}~\\
Given the MLP $N : \R^{m_0} \times \Theta \to \R^{m_l}$ defined in \S~\ref{mlp}, a parameter $\theta \in \Theta$ and a dataset $\{x_1,\ldots,x_n\}\subset \R^{m_0}$ of size $n \in \N+1$, the corresponding NTK matrix $K(\theta) \in \mathbb{S}^{n m_l}_+$ is defined blockwise as 
\[
K(\theta) = \left[ \frac{1}{n} K_\theta(x_{i_1},x_{i_2}) : i_1,i_2 \in [1:n] \right].
\]
\end{definition}
Note that with the block matrix of pointwise Jacobians $J(\theta) = [ \frac{1}{\sqrt{n}} \partial_\theta N(x_i,\theta) : i \in [1:n]] \in \Ell(\Theta, \R^{n m_l})$, we can write the NTK matrix as $K(\theta) = J(\theta) J(\theta)^*$.

\begin{remark}[Normalization factor]~\\
Note that there is a factor $\frac{1}{n}$ in the formula above. This is absent in most formulations but appears naturally when we consider the NTK matrix as an integral operator $K(\theta) \in \Ell(L^2(\mu,\R^{m_l}),L^2(\mu,\R^{m_l}))$ induced by the NTK with respect to the dataset considered as a probability measure $\mu = \frac{1}{n} \sum_{i=1}^n \delta_{x_i}$ (so that $L^2(\mu,\R^{m_l}) \cong \R^{n m_l}$).
\end{remark}

The limiting NTK for MLPs with $(a,b)$-ReLUs at the EOC takes the following form by \citet[Proposition~10]{mlpsateoc1}.

\begin{definition}[Limiting NTK]~\\
Define $\limiting{K} : \R^{m_0} \times \R^{m_0} \to \R^{m_l \times m_l}$ for all $x_1,x_2 \in \R^{m_0}$ as
\[
\limiting{K}(x_1,x_2) = \Vert x_1 \Vert \Vert x_2 \Vert \left( \sum_{k=1}^l \varrho^{\circ (k-1)}\left( \rho_1(x_1,x_2) \right) \prod_{k'=k}^{l-1} \varrho'\left( \varrho^{\circ (k'-1)}\left( \rho_1(x_1,x_2) \right) \right) \right) \Id_{m_l}.
\]
\end{definition}

\begin{definition}[Limiting NTK matrix]~\\
Given a dataset $\{x_1,\ldots,x_n\} \subset \R^{m_0}$ of size $n \in \N+1$, the corresponding limiting NTK matrix $\limiting{K} \in \mathbb{S}^{n m_l}_+$ is defined blockwise as 
\[
\limiting{K} = \left[ \frac{1}{n} \limiting{K}(x_{i_1},x_{i_2}) : i_1,i_2 \in [1:n] \right].
\]
\end{definition}

We are going to show that the components considered in \S~\ref{layerwise} corresponding to different layers and data points concentrate simultaneously, starting with the norms of the activations.

\begin{proposition}[Limiting concentration of norms of activations]\label{prop:limiting_norm_concentration}~\\
Given $p \in (0,1)$, a dataset $\{x_1,\ldots,x_n\} \subset \R^{m_0}$ of size $n \in \N+1$ and setting \eqref{eq:optimal_qs} and \eqref{eq:optimal_gammas}, for the event $E_1 \in \mathcal{B}(\Theta_{1:l-1})$ defined by having $\theta_{1:l-1} \in E_1$ iff
\[
\vert \Vert x_k(x_i, \theta_{1 : k-1}) \Vert - \Vert x_i \Vert \vert \leq \log(k) \Vert x_i \Vert O\left( \sqrt{\log(ln)} \kappa_\phi^2 m^{-\frac{1}{2}(1-p)} \right)
\]
for all $i \in [1:n]$ and $k \in [1:l]$, we have the bound $\Prob_{1:l-1}(E_1) \geq 1 - 2e^{-m^p}$.
\end{proposition}
\begin{proof}
Combining \citet[Lemma~2.2.2]{Vandervaartetal2023} with Proposition~\ref{prop:norm_concentration}, for all $t \geq 0$ we have 
\[
\Prob_{1:l-1}\left( \max_{k \in [2:l], i \in [1:n]}\left\{ (k-1) \left\vert \frac{\Vert x_k(x_i, \theta_{1 : k-1}) \Vert}{\Vert x_{k-1}(x_i, \theta_{1 : k-2}) \Vert} - 1 \right\vert \right\} \geq t \right)
\leq 2e^{-\frac{t^2}{O\left( \sqrt{\log(ln)} \kappa_\phi^2 m^{-\frac{1}{2}} \right)^2}}.
\]

Let $t = O(\sqrt{\log(ln)} \kappa_\phi^2 m^{-\frac{1}{2}}) m^{\frac{1}{2}p}$ and condition on the opposite event, happening with probability at least $1 - 2e^{-m^p}$. Denoting $\epsilon = O\left( \sqrt{\log(ln)} \kappa_\phi^2 m^{-\frac{1}{2}(1-p)} \right)$, we then have
\begin{equation}\label{eq:limiting_norm_concentration}
\left\vert \frac{\Vert x_k(x_i, \theta_{1 : k-1}) \Vert}{\Vert x_{k-1}(x_i, \theta_{1 : k-2}) \Vert} - 1 \right\vert \leq (k-1)^{-1} \epsilon
\end{equation}
for all $k \in [2:l]$ and $i \in [1:n]$.

As $x_1(x_i) = x_i$, we have $\Vert x_1(x_i) \Vert = \Vert x_i \Vert$ for all $i \in [1:n]$. Applying \eqref{eq:limiting_norm_concentration} inductively, we then have the bound $\vert \Vert x_k(x_i, \theta_{1 : k-1}) \Vert - \Vert x_i \Vert \vert \leq \log(k) \epsilon \Vert x_i \Vert$ for all $k \in [1:l]$ and $i \in [1:n]$.
\end{proof}

In the proof of the result below, we are going to use the law of cosines to get the concentration of cosine distances from the concentration of the corresponding proxies given by Proposition~\ref{prop:cosine_distance_concentration}. This will lead to the concentration of the inverse cosine distances, the propagation of which is determined by the inverse cosine distance map $\omega : (1,\infty) \to (1,\infty)$ defined as $\omega(z^{-\frac{1}{2}}) = \zeta(z)^{-\frac{1}{2}}$ for all $z \in (0,1)$ (see \citet[Proposition~12]{mlpsateoc1} for its properties).

\begin{proposition}[Limiting concentration of cosines of activations]\label{prop:limiting_cosine_concentration}~\\
Given $p \in (0,1)$, a dataset $\{ x_1,\cdots,x_n \} \subset \R^{m_0}$ of size $n \in \N+2$ with no parallel datapoints and setting \eqref{eq:optimal_qs} and \eqref{eq:optimal_gammas}, for the event $E_2 \in \mathcal{B}(\Theta_{1:l-1})$ defined by having $\theta_{1:l-1} \in E_2$ iff $\theta_{1:l-1} \in E_1$ and
\[
\left\vert \rho_k(x_{i_1},x_{i_2},\theta_{1:k-1}) - \varrho^{\circ(k-1)}(\rho_1(x_{i_1},x_{i_2})) \right\vert
\leq \Delta_\phi^{-2} (k-1)^{-2} O\left( \sqrt{\log(ln)} \kappa_\phi^2 m^{-\frac{1}{2}(1-p)} \right)
\]
for all $i_1 \neq i_2 \in [1:n]$ and $k \in [1:l]$, we have the bound $\Prob_{1:l-1}(E_2) \geq 1-4e^{-m^p}$.
\end{proposition}
\begin{proof}
Combining \citet[Lemma~2.2.2]{Vandervaartetal2023} with Proposition~\ref{prop:cosine_distance_concentration}, for all $t \geq 0$ we have
\begin{multline*}
\Prob\left( \max_{k \in [2:l], i_1 \neq i_2 \in [1:n]}\left\{ \zeta(z_{k-1,i_1,i_2})^{-\frac{1}{2}} \frac{\zeta(z_{k-1,i_1,i_2})}{z_{k-1,i_1,i_2}} (k-1) \left\vert \left\Vert \frac{1}{2} \frac{x_k(x_{i_1}, \theta_{1 : k-1})}{\Vert x_{k-1}(x_{i_1}, \theta_{1 : k-2}) \Vert} 
\right. \right. \right. \right. \\ \left. \left. \left. \left.
- \frac{1}{2} \frac{x_k(x_{i_2}, \theta_{1 : k-1})}{\Vert x_{k-1}(x_{i_2}, \theta_{1 : k-2}) \Vert} \right\Vert - \zeta(z_{k-1,i_1,i_2})^{\frac{1}{2}} \right\vert \right\} \geq t \right)
\leq 2e^{-\frac{t^2}{O\left( \sqrt{\log(ln)} \kappa_\phi^2 m^{-\frac{1}{2}} \right)^2}}
\end{multline*}
with $z_{k,i_1,i_2} = \frac{1 - \rho_k(x_{i_1},x_{i_2},\theta_{1:k-1})}{2} = \left\Vert \frac{1}{2} \frac{x_k(x_{i_1}, \theta_{1 : k-1})}{\Vert x_k(x_{i_1}, \theta_{1 : k-1}) \Vert} - \frac{1}{2} \frac{x_k(x_{i_2}, \theta_{1 : k-1})}{\Vert x_k(x_{i_2}, \theta_{1 : k-1}) \Vert} \right\Vert^2$.

Let $t = O(\sqrt{\log(ln)} \kappa_\phi^2 m^{-\frac{1}{2}}) m^{\frac{1}{2}p}$ and condition on the opposite of this event and the event of Proposition~\ref{prop:limiting_norm_concentration}, happening at the same time with probability at least $1 - 4e^{-m^p}$ via a Fr\'echet bound. Denoting $\epsilon = O\left( \sqrt{\log(ln)} \kappa_\phi^2 m^{-\frac{1}{2}(1-p)} \right)$ (with the implicit constant changing from time to time), we then have
\begin{multline}\label{eq:limiting_cosine_distance_concentration}
\left\vert \left\Vert \frac{1}{2} \frac{x_k(x_{i_1}, \theta_{1 : k-1})}{\Vert x_{k-1}(x_{i_1}, \theta_{1 : k-2}) \Vert} - \frac{1}{2} \frac{x_k(x_{i_2}, \theta_{1 : k-1})}{\Vert x_{k-1}(x_{i_2}, \theta_{1 : k-2}) \Vert} \right\Vert - \zeta(z_{k-1,i_1,i_2})^{\frac{1}{2}} \right\vert \\
\leq \zeta(z_{k-1,i_1,i_2})^{\frac{1}{2}} \frac{z_{k-1,i_1,i_2}}{\zeta(z_{k-1,i_1,i_2})} (k-1)^{-1} \epsilon
\end{multline}
for all $k \in [2:l]$ and $i_1 \neq i_2 \in [1:n]$.

Note now that the cosine of the angle enclosed by the first two sides of the triangle with sides 
\[
\left( 
\frac{\Vert x_k(x_{i_1}, \theta_{1 : k-1}) \Vert}{\Vert x_{k-1}(x_{i_1}, \theta_{1 : k-2}) \Vert}, 
\frac{\Vert x_k(x_{i_2}, \theta_{1 : k-1}) \Vert}{\Vert x_{k-1}(x_{i_2}, \theta_{1 : k-2}) \Vert}, 
\left\Vert \frac{x_k(x_{i_1}, \theta_{1 : k-1})}{\Vert x_{k-1}(x_{i_1}, \theta_{1 : k-2}) \Vert} - \frac{x_k(x_{i_2}, \theta_{1 : k-1})}{\Vert x_{k-1}(x_{i_2}, \theta_{1 : k-2}) \Vert} \right\Vert
\right)
\]
is exactly $\rho_k(x_{i_1},x_{i_2},\theta_{1:k-1})$. By the law of cosines, we then have that
\begin{multline*}
z_{k,i_1,i_2}^{\frac{1}{2}}
= \left\Vert \frac{1}{2} \frac{x_k(x_{i_1}, \theta_{1 : k-1})}{\Vert x_k(x_{i_1}, \theta_{1 : k-1}) \Vert} - \frac{1}{2} \frac{x_k(x_{i_2}, \theta_{1 : k-1})}{\Vert x_k(x_{i_2}, \theta_{1 : k-1}) \Vert} \right\Vert
= \sqrt{ \frac{1 - \rho_k(x_{i_1},x_{i_2},\theta_{1:k-1})}{2} } \\
= \sqrt{ \frac{ 4 \left\Vert \frac{1}{2} \frac{x_k(x_{i_1}, \theta_{1 : k-1})}{\Vert x_{k-1}(x_{i_1}, \theta_{1 : k-2}) \Vert} - \frac{1}{2} \frac{x_k(x_{i_2}, \theta_{1 : k-1})}{\Vert x_{k-1}(x_{i_2}, \theta_{1 : k-2}) \Vert} \right\Vert^2 - \frac{\Vert x_k(x_{i_1}, \theta_{1 : k-1}) \Vert^2}{\Vert x_{k-1}(x_{i_1}, \theta_{1 : k-2}) \Vert^2} - \frac{\Vert x_k(x_{i_2}, \theta_{1 : k-1}) \Vert^2}{\Vert x_{k-1}(x_{i_2}, \theta_{1 : k-2}) \Vert^2} }{ 4 \frac{\Vert x_k(x_{i_1}, \theta_{1 : k-1}) \Vert}{\Vert x_{k-1}(x_{i_1}, \theta_{1 : k-2}) \Vert} \frac{\Vert x_k(x_{i_2}, \theta_{1 : k-1}) \Vert}{\Vert x_{k-1}(x_{i_2}, \theta_{1 : k-2}) \Vert} } + \frac{1}{2} }.
\end{multline*}
By \eqref{eq:limiting_norm_concentration} and \eqref{eq:limiting_cosine_distance_concentration}, we have the bounds
\[
1-(k-1)^{-1}\epsilon
\leq \frac{\Vert x_k(x_i, \theta_{1 : k-1}) \Vert}{\Vert x_{k-1}(x_i, \theta_{1 : k-2}) \Vert}
\leq 1+(k-1)^{-1}\epsilon
\]
for all $i \in [1:n]$ and
\begin{multline*}
\left( 1 - \frac{z_{k-1,i_1,i_2}}{\zeta(z_{k-1,i_1,i_2})} (k-1)^{-1} \epsilon \right) \zeta(z_{k-1,i_1,i_2})^{\frac{1}{2}}
\leq \left\Vert \frac{1}{2} \frac{x_k(x_{i_1}, \theta_{1 : k-1})}{\Vert x_{k-1}(x_{i_1}, \theta_{1 : k-2}) \Vert} - \frac{1}{2} \frac{x_k(x_{i_2}, \theta_{1 : k-1})}{\Vert x_{k-1}(x_{i_2}, \theta_{1 : k-2}) \Vert} \right\Vert \\
\leq \left( 1 + \frac{z_{k-1,i_1,i_2}}{\zeta(z_{k-1,i_1,i_2})} \epsilon (k-1)^{-1} \right) \zeta(z_{k-1,i_1,i_2})^{\frac{1}{2}}.
\end{multline*}
Therefore we have
\[
\frac{1 - \frac{z_{k-1,i_1,i_2}}{\zeta(z_{k-1,i_1,i_2})} (k-1)^{-1} \epsilon}{1+(k-1)^{-1}\epsilon} \zeta(z_{k-1,i_1,i_2})^{\frac{1}{2}}
\leq z_{k,i_1,i_2}^{\frac{1}{2}}
\leq \frac{1 + \frac{z_{k-1,i_1,i_2}}{\zeta(z_{k-1,i_1,i_2})} (k-1)^{-1} \epsilon}{1-(k-1)^{-1}\epsilon} \zeta(z_{k-1,i_1,i_2})^{\frac{1}{2}},
\]
so that
\[
\frac{1-(k-1)^{-1}\epsilon}{1 + \frac{\omega(w_{k-1,i_1,i_2})^2}{w_{k-1,i_1,i_2}^2} (k-1)^{-1} \epsilon} \omega(w_{k-1,i_1,i_2})
\leq w_{k,i_1,i_2}
\leq \frac{1+(k-1)^{-1}\epsilon}{1 - \frac{\omega(w_{k-1,i_1,i_2})^2}{w_{k-1,i_1,i_2}^2} (k-1)^{-1} \epsilon} \omega(w_{k-1,i_1,i_2})
\]
with $w_{k,i_1,i_2} = z_{k,i_1,i_2}^{-\frac{1}{2}}$. Subtracting $\omega(w_{k-1,i_1,i_2})$, we get a bound that implies
\[
\vert w_{k,i_1,i_2} - \omega(w_{k-1,i_1,i_2}) \vert
\leq \frac{\left( 1 + \frac{\omega(w_{k-1,i_1,i_2})^2}{w_{k-1,i_1,i_2}^2} \right) (k-1)^{-1} \epsilon}{1 - \frac{\omega(w_{k-1,i_1,i_2})^2}{w_{k-1,i_1,i_2}^2} (k-1)^{-1} \epsilon} \omega(w_{k-1,i_1,i_2}).
\]
Assume now that $\vert w_{k',i_1,i_2} - \omega^{\circ (k'-1)}(w_{1,i_1,i_2}) \vert \leq \Delta_\phi (k'-1) \epsilon$ for $k' \in [1:k-1]$, which clearly holds for $k=2$. Then 
\begin{multline*}
w_{1,i_1,i_2} + \Delta_\phi \left( \frac{4}{3\pi} - \epsilon \right) (k'-1) + \Delta_\phi \frac{2}{\pi} \log\left( \Delta_\phi^{-1} \frac{3\pi}{4} w_{1,i_1,i_2} + k' - 1 \right) - O(1) 
\leq w_{k',i_1,i_2} \\
\leq w_{1,i_1,i_2} + \Delta_\phi \left( \frac{4}{3\pi} + \epsilon \right) (k'-1) + \Delta_\phi \frac{2}{\pi} \log\left( \Delta_\phi^{-1} \frac{3\pi}{4} w_{1,i_1,i_2} + k' - 1 \right) + O(1)
\end{multline*}
by \citet[Proposition~13]{mlpsateoc1}, so that $\frac{\omega(w_{k-1,i_1,i_2})^2}{w_{k-1,i_1,i_2}^2} \leq O(1)$ and $(k-1)^{-1} \omega(w_{k-1,i_1,i_2}) \leq O(\Delta_\phi)$. We then have $\vert w_{k,i_1,i_2} - \omega(w_{k-1,i_1,i_2}) \vert \leq \frac{\Delta_\phi \epsilon}{1-(k-1)^{-1} \epsilon} \leq \Delta_\phi \epsilon$, so that by the triangle inequality and using that $\omega$ is $1$-Lipschitz by \citet[Proposition~12]{mlpsateoc1} we have the bound 
\begin{multline*}
\vert w_{k,i_1,i_2} - \omega^{\circ (k-1)}(w_{1,i_1,i_2}) \vert 
\leq \vert w_{k,i_1,i_2} - \omega(w_{k-1,i_1,i_2}) \vert + \vert \omega(w_{k-1,i_1,i_2}) - \omega^{\circ (k-1)}(w_{1,i_1,i_2}) \vert \\
\leq \vert w_{k,i_1,i_2} - \omega(w_{k-1,i_1,i_2}) \vert + \vert w_{k-1,i_1,i_2} - \omega^{\circ (k-2)}(w_{1,i_1,i_2}) \vert 
\leq \Delta_\phi \epsilon + \Delta_\phi (k-2) \epsilon 
\leq \Delta_\phi (k-1),
\end{multline*}
completing the induction. Hence $\vert w_{k,i_1,i_2} - \omega^{\circ (k-1)}(w_{1,i_1,i_2}) \vert \leq \Delta_\phi (k-1) \epsilon$ for all $k \in [1:l]$ and $i_1 \neq i_2 \in [1:n]$.

Given $i_1 \neq i_2 \in [1:n]$, as $\rho_k(x_{i_1},x_{i_2},\theta_{1:k-1}) = 1-2w_{k,i_1,i_2}^{-2}$ and $\varrho^{\circ(k-1)}(\rho_1(x_{i_1},x_{i_2}) = 1-2\omega^{\circ (k-1)}(w_{1,i_1,i_2})^{-2}$, by the fundamental theorem of calculus we have $\vert \rho_k(x_{i_1},x_{i_2},\theta_{1:k-1}) - \varrho^{\circ(k-1)}(\rho_1(x_{i_1},x_{i_2}) \vert \leq O((\Delta_\phi(k-1))^{-3}) \Delta_\phi (k-1) \epsilon \leq \Delta_\phi^{-2} (k-1)^{-2} \epsilon$ for all $i_1 \neq i_2 \in [1:n]$ and $k \in [1:l]$.
\end{proof}

\begin{proposition}[Limiting concentration of norms of backpropagation matrices]\label{prop:limiting_backprop_norm_concentration}~\\
Given $p \in (0,1)$, a dataset $\{x_1,\cdots,x_n\} \subset \R^{m_0}$ of size $n \in \N+1$ and setting \eqref{eq:optimal_qs} and \eqref{eq:optimal_gammas}, for the event $E_3 \in \mathcal{B}(\Theta_{1:l-1})$ defined by having $\theta_{1:l-1} \in E_3$ iff
\[
\Vert B_{k_1, k_2}(x_i, \theta_{1 : k_2-1}) \Vert \leq \log\left( \frac{k_2-1}{k_1-1} \right) O\left( \kappa^2 m^{-\frac{1}{2}} \right)
\]
and
\[
\left\vert \Vert B_{k_1,k_2}(x_i, \theta_{1 : k_2-1}) \Vert_F^2 - 1 \right\vert
\leq \log\left( \frac{k_2}{k_1-1} \right) O\left( \sqrt{\log(ln)} \kappa_\phi^2 m^{-\frac{1}{2}(1-p)} \right)
\]
for all $i \in [1:n]$ and $k_1 \leq k_2 \in [2:l]$, we have the bound $\Prob_{1:l-1}(E_3) \geq 1 - 6e^{-m^p}$.
\end{proposition}
\begin{proof}
Combining \citet[Lemma~2.2.2]{Vandervaartetal2023} with Proposition~\ref{prop:backprop_norm_concentration}, for all $t_1 \geq 0$ we have
\begin{multline*}
\Prob_{1:l-1}\left(
\max_{k_1 < k_2 \in [2:l], i \in [1:n]}\left\{ (k_2-1) \left( 
\vphantom{\frac{\Vert B_{k_1, k_2}(x_i, \theta_{1 : k_2-1}) \Vert - O\left( \kappa^2 (k_2-1)^{-1} m^{-\frac{1}{2}} \Vert B_{k_1, k_2-1}(x_i, \theta_{1 : k_2-2}) \Vert_F \right)}{\Vert B_{k_1, k_2-1}(x_i, \theta_{1 : k_2-2}) \Vert}} \right. \right. \right. \\ \left. \left. \left.
\frac{\Vert B_{k_1, k_2}(x_i, \theta_{1 : k_2-1}) \Vert - O\left( \kappa^2 (k_2-1)^{-1} m^{-\frac{1}{2}} \Vert B_{k_1, k_2-1}(x_i, \theta_{1 : k_2-2}) \Vert_F \right)}{\Vert B_{k_1, k_2-1}(x_i, \theta_{1 : k_2-2}) \Vert} - 1 \right)_+ \right\} \geq t_1
\right) \\
\leq 2e^{-\frac{t_1^2}{O\left( \sqrt{\log(ln)} \kappa_\phi^2 m^{-\frac{1}{2}} \right)^2}},
\end{multline*}
while combining \citet[Lemma~2.2.13]{Vandervaartetal2023} with Proposition~\ref{prop:bwd_concentration_1} and Proposition~\ref{prop:bwd_concentration_2}, for all $t_2 \geq 0$ we have
\begin{multline*}
\Prob_{1:l-1}\left( \max_{k \in [2:l], i \in [1:n]}\left\{ (k-1) \left\vert \Vert B_{k,k}(x_i, \theta_{1 : k-1}) \Vert_F^2 - 1 \right\vert \right\} 
\right. \\ \left.
\geq O\left( \sqrt{\log(ln)} \kappa_\phi^2 m^{-\frac{1}{2}} \right) t_1^{\frac{1}{2}} 
+ O\left( \log(ln) \kappa_\phi^2 m^{-1} \right) t_2 \right)
\leq 2e^{-t_2}
\end{multline*}
and
\begin{multline*}
\Prob_{1:l-1}\left( \max_{k_1 < k_2 \in [2:l], i \in [1:n]}\left\{ (k_2-1) \left\vert \frac{\Vert B_{k_1,k_2}(x_i,\theta_{1:k_2-1}) \Vert_F^2}{\Vert B_{k_1,k_2-1}(x_i,\theta_{1:k_2-2}) \Vert_F^2} - 1 \right\vert \right\} 
\right. \\ \left. \vphantom{\left\{ (k_2-1) \left\vert \frac{\Vert B_{k_1,k_2}(x_i,\theta_{1:k_2-1}) \Vert_F^2}{\Vert B_{k_1,k_2-1}(x_i,\theta_{1:k_2-2}) \Vert_F^2} - 1 \right\vert \right\}}
\geq O\left( \sqrt{\log(ln)} \kappa_\phi^2 m^{-\frac{1}{2}} \right) t_2^{\frac{1}{2}} 
+ O\left( \log(ln) \kappa_\phi^2 m^{-1} \right) t_2 \right)
\leq 2e^{-t_2}.
\end{multline*}

Let $t_1 = O(\sqrt{\log(ln)} \kappa_\phi^2 m^{-\frac{1}{2}}) m^{\frac{1}{2}p}$, $t_2 = m^p$ and condition on the opposites of the above events, happening with probability at least $1 - 6e^{-m^p}$. Denoting $\epsilon = O\left( \sqrt{\log(ln)} \kappa_\phi^2 m^{-\frac{1}{2}(1-p)} \right)$ (with the implicit constant changing from time to time), we then have
\begin{multline}\label{eq:limiting_backprop_norm_concentration}
\Vert B_{k_1, k_2}(x_i, \theta_{1 : k_2-1}) \Vert 
\leq O\left( \kappa_\phi^2 (k_2-1)^{-1} m^{-\frac{1}{2}} \Vert B_{k_1, k_2-1}(x_i, \theta_{1 : k_2-2}) \Vert_F \right) \\
+ \left( 1 + (k_2-1)^{-1} \epsilon \right) \Vert B_{k_1, k_2-1}(x_i, \theta_{1 : k_2-2}) \Vert
\end{multline}
for all $k_1 < k_2 \in [2:l]$ and $i \in [1:n]$, 
\begin{equation}\label{eq:limiting_backprop_fro_concentration_1}
\left\vert \Vert B_{k,k}(x_i, \theta_{1 : k-1}) \Vert_F^2 - 1 \right\vert
\leq (k-1)^{-1} \epsilon
\end{equation}
for all $k \in [2:l]$ and $i \in [1:n]$ and
\begin{equation}\label{eq:limiting_backprop_fro_concentration_2}
\left\vert \frac{\Vert B_{k_1,k_2}(x_i, \theta_{1 : k_2-1}) \Vert_F^2}{\Vert B_{k_1,k_2-1}(x_i, \theta_{1 : k_2-2}) \Vert_F^2} - 1 \right\vert
\leq (k_2-1)^{-1} \epsilon
\end{equation}
for all $k_1 < k_2 \in [2:l]$ and $i \in [1:n]$.

Applying \eqref{eq:limiting_backprop_fro_concentration_1} and \eqref{eq:limiting_backprop_fro_concentration_2} inductively, we have the bound
\[
\left\vert \Vert B_{k_1,k_2}(x_i, \theta_{1 : k_2-1}) \Vert_F^2 - 1 \right\vert
\leq \log\left( \frac{k_2}{k_1-1} \right) \epsilon
\]
for all $k_1 < k_2 \in [2:l]$ and $i \in [1:n]$, so that $\Vert B_{k_1,k_2}(x_i, \theta_{1 : k_2-1}) \Vert_F^2 \leq O(1)$.

For all $k \in [2:l]$ and $i \in [1:n]$, $\Vert B_{k,k}(x_i,\theta_{1:k-1}) \Vert = \Vert \sigma D_{x_k'(x_i,\theta_{1:k-1})} \Vert$ is bounded by $\sigma m_{k-1}^{-\frac{1}{2}} \Vert \phi \Vert_L = \kappa_\phi (k-1)^{-1} m^{-\frac{1}{2}}$. Assume now that $\Vert B_{k_1,k_2'}(x_i,\theta_{1:k_2'-1}) \Vert \leq ((k_1-1)^{-1} + O(\kappa_\phi \sum_{k=k_1}^{k_2'-1} k^{-1}) + \epsilon) \kappa_\phi m^{-\frac{1}{2}}$ for $k_2' \in [k_1:k_2-1]$ and $i \in [1:n]$ for some $k_1 \leq k_2 \in [2:l]$, which clearly holds if $k_1=k_2-1$. Then by \eqref{eq:limiting_backprop_norm_concentration}, $\Vert B_{k_1,k_2}(x_i,\theta_{1:k_2-1}) \Vert$ is at most
\begin{multline*}
O(\kappa_\phi^2 (k_2-1)^{-1} m^{-\frac{1}{2}}) + (1 + (k_2-1)^{-1} \epsilon) \left( (k_1-1)^{-1} + O\left( \kappa_\phi \sum_{k=k_1}^{k_2-2} k^{-1} \right) + \epsilon \right) \kappa_\phi m^{-\frac{1}{2}} \\
\leq \left( (k_1-1)^{-1} + O\left( \kappa_\phi \sum_{k=k_1}^{k_2-1} k^{-1} \right) + \epsilon \right) \kappa_\phi m^{-\frac{1}{2}},
\end{multline*}
completing the induction. Hence $\Vert B_{k_1,k_2}(x_i,\theta_{1:k_2-1}) \Vert$ is at most
\[
\left( (k_1-1)^{-1} + O\left( \kappa_\phi \sum_{k=k_1}^{k_2-1} k^{-1} \right) + \epsilon \right) \kappa_\phi m^{-\frac{1}{2}}
\leq O\left( \kappa_\phi^2 \log\left( \frac{k_2-1}{k_1-1} \right) m^{-\frac{1}{2}} \right)
\]
for all $k_1 \leq k_2 \in [2:l]$ and $i \in [1:n]$.
\end{proof}

\begin{proposition}[Limiting concentration of backpropagation inner products]\label{prop:limiting_backprop_inner_product_concentration}~\\
Given $p \in (0,1)$, a dataset $\{x_1,\cdots,x_n\} \subset \R^{m_0}$ of size $n \in \N+2$ with no parallel datapoints and setting \eqref{eq:optimal_qs} and \eqref{eq:optimal_gammas}, for the event $E_4 \in \mathcal{B}(\Theta_{1:l-1})$ defined by having $\theta_{1:l-1} \in E_4$ iff $\theta_{1:l-1} \in E_2 \cap E_3$ and
\begin{multline*}
\left\vert X'_{k_1,k_2}(x_{i_1},x_{i_2},\theta_{1:k_2-1}) - \prod_{k=k_1}^{k_2} \varrho'(\varrho^{\circ (k-2)}(\rho_1(x_{i_1},x_{i_2}))) \right\vert \\
\leq \log\left( \frac{k_2}{k_1-1} \right) O\left( \sqrt{\log(ln)} \kappa_\phi^2 m^{-\frac{1}{2}(1-p)} \right)
\end{multline*}
for all $i_1, i_2 \in [1:n]$ and $k_1 \leq k_2 \in [2:l]$, we have the bound $\Prob_{1:l-1}(E_4) \geq 1-14e^{-m^p}$.
\end{proposition}
\begin{proof}
Combining \citet[Lemma~2.2.13]{Vandervaartetal2023} with Proposition~\ref{prop:bwd_concentration_1} and Proposition~\ref{prop:bwd_concentration_2}, for all $t \geq 0$ we have
\begin{multline*}
\Prob_{1:l-1}\left( \max_{k \in [2:l], i_1 \neq i_2 \in [1:n]}\left\{ (k-1) \left\vert X'_{k,k}(x_{i_1},x_{i_2},\theta_{1:k-1}) - \varrho'(\rho_{k-1}(x_{i_1},x_{i_2},\theta_{1:k-2})) \right\vert \right\} 
\right. \\ \left.
\geq O\left( \sqrt{\log(ln)} \kappa_\phi^2 m^{-\frac{1}{2}} \right) t^{\frac{1}{2}} 
+ O\left( \log(ln) \kappa_\phi^2 m^{-1} \right) t \right)
\leq 2e^{-t}
\end{multline*}
and
\begin{multline*}
\Prob_{1:l-1}\left( \max_{k_1 < k_2 \in [2:l], i_1 \neq i_2 \in [1:n]}\left\{ (k_2-1) 
\vphantom{\frac{\left\vert X'_{k_1,k_2}(x_{i_1},x_{i_2},\theta_{1:k_2-1}) - \E_{A_{k_2-1}} X'_{k_1,k_2}(x_{i_1},x_{i_2},\theta_{1:k_2-1}) \right\vert}{\sqrt{X'_{k_1,k_2-1}(x_{i_1},x_{i_1},\theta_{1:k_2-2})} \sqrt{X'_{k_1,k_2-1}(x_{i_2},x_{i_2},\theta_{1:k_2-2})}}} \right. \right. \\ \left. \left.
\frac{\left\vert X'_{k_1,k_2}(x_{i_1},x_{i_2},\theta_{1:k_2-1}) - \E_{A_{k_2-1}} X'_{k_1,k_2}(x_{i_1},x_{i_2},\theta_{1:k_2-1}) \right\vert}{\sqrt{X'_{k_1,k_2-1}(x_{i_1},x_{i_1},\theta_{1:k_2-2})} \sqrt{X'_{k_1,k_2-1}(x_{i_2},x_{i_2},\theta_{1:k_2-2})}} \right\} 
\right. \\ \left. 
\geq O\left( \sqrt{\log(ln)} \kappa_\phi^2 m^{-\frac{1}{2}} \right) t^{\frac{1}{2}} 
+ O\left( \log(ln) \kappa_\phi^2 m^{-1} \right) t \right)
\leq 2e^{-t}.
\end{multline*}

Let $t = m^p$ and condition on the opposites of these events and the events of Proposition~\ref{prop:limiting_cosine_concentration} and Proposition~\ref{prop:limiting_backprop_norm_concentration}, happening at the same time with probability at least $1 - 14e^{-m^p}$ via a Fr\'echet bound. Denoting $\epsilon = O\left( \sqrt{\log(ln)} \kappa_\phi^2 m^{-\frac{1}{2}(1-p)} \right)$ (with the implicit constant changing from time to time), we then have
\begin{equation}\label{eq:limiting_bwd_1_concentration}
\left\vert X'_{k,k}(x_{i_1},x_{i_2},\theta_{1:k-1}) - \varrho'(\rho_{k-1}(x_{i_1},x_{i_2},\theta_{1:k-2})) \right\vert
\leq (k-1)^{-1} \epsilon
\end{equation}
for all $k \in [2:l]$ and $i_1, i_2 \in [1:n]$ and
\begin{multline}\label{eq:limiting_bwd_2_concentration}
\left\vert X'_{k_1,k_2}(x_{i_1},x_{i_2},\theta_{1:k_2-1}) - \E_{A_{k_2-1}} X'_{k_1,k_2}(x_{i_1},x_{i_2},\theta_{1:k_2-1}) \right\vert \\
\leq \sqrt{X'_{k_1,k_2-1}(x_{i_1},x_{i_1},\theta_{1:k_2-2})} \sqrt{X'_{k_1,k_2-1}(x_{i_2},x_{i_2},\theta_{1:k_2-2})} (k_2-1)^{-1} \epsilon
\end{multline}
for $k_1 < k_2 \in [2:l]$ and $i_1, i_2 \in [1:n]$.

As $\varrho''(1-2w^{-2}) = -\frac{1}{2} \zeta''(w^{-2}) = \Delta_\phi \frac{2}{\pi} (1-(1-2w^{-2})^2)^{-\frac{1}{2}} = \Delta_\phi \frac{1}{\pi} (1-w^{-2})^{-\frac{1}{2}} w$ by \citet[Proposition~11]{mlpsateoc1}, by \eqref{eq:limiting_bwd_1_concentration} we have 
\begin{multline*}
\vert X'_{k,k}(x_{i_1},x_{i_2},\theta_{1:k-1}) - \varrho'(\varrho^{\circ (k-2)}(\rho_1(x_{i_1},x_{i_2}))) \vert \\
\leq \vert X'_{k,k}(x_{i_1},x_{i_2},\theta_{1:k-1}) - \varrho'(\rho_{k-1}(x_{i_1},x_{i_2},\theta_{1:k-2})) \vert \\
+ \vert \varrho'(\rho_{k-1}(x_{i_1},x_{i_2},\theta_{1:k-2})) - \varrho'(\varrho^{\circ (k-2)}(\rho_1(x_{i_1},x_{i_2}))) \vert \\
\leq (k-1)^{-1} \epsilon + O(\Delta_\phi^2 (k-2)) \Delta_\phi^{-2} (k-2)^{-2} \epsilon 
\leq (k-1)^{-1} \epsilon
\end{multline*}
for all $k \in [2:l]$ and $i_1 \neq i_2 \in [1:n]$.

By Proposition~\ref{prop:bwd_expectation_offdiagonal}, for $i_1 \neq i_2 \in [1:n]$ and $k_1 < k_2 \in [2:l]$ we then have
\begin{multline*}
\vert \E_{A_{k_2-1}} X'_{k_1,k_2}(x_{i_1},x_{i_2},\theta_{1:k_2-1}) - \varrho'(\rho_{k_2-1}(x_{i_1},x_{i_2},\theta_{1:k_2-2})) X'_{k_1,k_2-1}(x_{i_1},x_{i_2},\theta_{1:k_2-2}) \vert \\
\leq \Delta_\phi \frac{8}{\pi} \sqrt{\frac{1-\rho_{k_2-1}(x_{i_1},x_{i_2},\theta_{1:k_2-2})}{1+\rho_{k_2-1}(x_{i_1},x_{i_2},\theta_{1:k_2-2})}} \Vert B_{k_1,k_2-1}(x_{i_1},\theta_{1:k_2-2}) \Vert \Vert B_{k_1,k_2-1}(x_{i_2},\theta_{1:k_2-2}) \Vert \\
\leq O((k_2-1)^{-1}) O\left( \kappa_\phi^2 \log\left( \frac{k_2-2}{k_1-1} \right) m^{-\frac{1}{2}} \right)^2
\leq O\left( \kappa_\phi^4 \log\left( \frac{k_2-2}{k_1-1} \right)^2 (k_2-1)^{-1} m^{-1} \right).
\end{multline*}

For $k_1 < k_2 \in [2:l]$, assume that $\vert X'_{k_1,k_2'}(x_{i_1},x_{i_2},\theta_{1:k_2'-1}) - \prod_{k=k_1}^{k_2'} \varrho'(\varrho^{\circ (k-2)}(\rho_1(x_{i_1},x_{i_2}))) \vert$ is at most $(\sum_{k=k_1-1}^{k_2'-1} k^{-1}) \epsilon + O(\kappa_\phi^4 (\sum_{k=k_1}^{k_2'-1} \log(\frac{k-1}{k_1-1})^2 k^{-1}) m^{-1})$ for $k_2' \in [k_1:k_2-1]$, which clearly holds for $k_2 = k_1+1$. By the triangle inequality, we have
\begin{multline*}
\left\vert X'_{k_1,k_2}(x_{i_1},x_{i_2},\theta_{1:k_2-1}) - \prod_{k=k_1}^{k_2} \varrho'(\varrho^{\circ (k-2)}(\rho_1(x_{i_1},x_{i_2}))) \right\vert \\
\leq \left\vert X'_{k_1,k_2}(x_{i_1},x_{i_2},\theta_{1:k_2-1}) - \E_{A_{k_2-1}} X'_{k_1,k_2}(x_{i_1},x_{i_2},\theta_{1:k_2-1}) \right\vert \\
+ \left\vert \E_{A_{k_2-1}} X'_{k_1,k_2}(x_{i_1},x_{i_2},\theta_{1:k_2-1}) - \varrho'(\rho_{k_2-1}(x_{i_1},x_{i_2},\theta_{1:k_2-2})) X'_{k_1,k_2-1}(x_{i_1},x_{i_2},\theta_{1:k_2-2}) \right\vert \\
+ \left\vert \varrho'(\rho_{k_2-1}(x_{i_1},x_{i_2},\theta_{1:k_2-2})) X'_{k_1,k_2-1}(x_{i_1},x_{i_2},\theta_{1:k_2-2})
\right. \\ \left.
- \varrho'(\varrho^{\circ (k_2-2)}(\rho_1(x_{i_1},x_{i_2}))) X'_{k_1,k_2-1}(x_{i_1},x_{i_2},\theta_{1:k_2-2}) \right\vert \\
+ \left\vert \varrho'(\varrho^{\circ (k_2-2)}(\rho_1(x_{i_1},x_{i_2}))) X'_{k_1,k_2-1}(x_{i_1},x_{i_2},\theta_{1:k_2-2}) - \prod_{k=k_1}^{k_2} \varrho'(\varrho^{\circ (k-2)}(\rho_1(x_{i_1},x_{i_2}))) \right\vert,
\end{multline*}
which is at most
\begin{multline*}
(k_2-1)^{-1} \epsilon
+ O\left( \kappa_\phi^4 \log\left( \frac{k_2-2}{k_1-1} \right)^2 (k_2-1)^{-1} m^{-1} \right) 
+ O\left( \Delta_\phi^2 (k_2-2) \right) \Delta_\phi^{-2} (k_2-2)^{-2} \epsilon \\
+ \left( \sum_{k=k_1-1}^{k_2-2} k^{-1} \right) \epsilon
+ O\left( \kappa_\phi^4 \left( \sum_{k=k_1}^{k_2-2} \log\left( \frac{k-1}{k_1-1} \right)^2 k^{-1} \right) m^{-1} \right) \\
\leq \left( \sum_{k=k_1-1}^{k_2-1} k^{-1} \right) \epsilon
+ O\left( \kappa_\phi^4 \left( \sum_{k=k_1}^{k_2-1} \log\left( \frac{k-1}{k_1-1} \right)^2 k^{-1} \right) m^{-1} \right).
\end{multline*}
By induction, we therefore have $\vert X'_{k_1,k_2}(x_{i_1},x_{i_2},\theta_{1:k_2-1}) - \prod_{k=k_1}^{k_2} \varrho'(\varrho^{\circ (k-2)}(\rho_1(x_{i_1},x_{i_2}))) \vert \leq \log(\frac{k_2}{k_1-1}) \epsilon + O(\kappa^4 \log(\frac{k_2}{k_1-1})^3 m^{-1}) \leq \log(\frac{k_2}{k_1-1}) \epsilon$ for $k_1 < k_2 \in [2:l]$ and $i_1 \neq i_2 \in [1:n]$.
\end{proof}

\begin{theorem}[Limiting concentration of $K(\theta)$]\label{thm:ntk_concentration}~\\
Given $p \in (0,1)$, a dataset $\{x_1,\cdots,x_n\} \subset \R^{m_0}$ of size $n \in \N+2$ with no parallel data points and setting \eqref{eq:optimal_qs} and \eqref{eq:optimal_gammas}, we have that
\[
\Prob\left( \left\Vert K(\theta) - \limiting{K} \right\Vert \leq O\left( \overline{\tau}^2 \left( \Delta_\phi^{-2} + \left( \log(l) + m_l^{\frac{1}{2}} \right) l \right) \sqrt{\log(ln)} \kappa_\phi^2 m^{-\frac{1}{2}(1-p)} \right) \right)
\]
is at least $1 - 16e^{-m^p}$ with $\overline{\tau} = \max_{i \in [1:n]}\left\{ \Vert x_i \Vert \right\}$.
\end{theorem}
\begin{proof}
First, assume that $\Vert x_i \Vert = 1$ for all $i \in [1:n]$.

Condition on the event of Proposition~\ref{prop:limiting_backprop_inner_product_concentration} happening with probability at least $1-14e^{-m^p}$ with respect to $\theta_{1:l-1}$ and denote $\epsilon = O\left( \sqrt{\log(ln)} \kappa_\phi^2 m^{-\frac{1}{2}(1-p)} \right)$ (with the implicit constant changing from time to time). Note that $\vert X_k(x_i,x_i,\theta_{1:k-1}) X'_{k+1,l}(x_i,x_i,\theta_{1 : l-1}) - 1 \vert \leq \log(k) \epsilon + \log(\frac{l}{k}) \epsilon \leq \log(l) \epsilon$, so that $\vert \sum_{k=1}^{l-1} X_k(x_i,x_i,\theta_{1:k-1}) X'_{k+1,l}(x_i,x_i,\theta_{1:l-1}) + X_l(x_i,x_i,\theta_{1:l-1}) - l \vert \leq \log(l) l \epsilon$. By Proposition~\ref{prop:readout_expectation}, this quantity also bounds $\Vert \E_{A_l} K_\theta(x_i,x_i) - l \Id_{m_l} \Vert$. Now also that for $k \in [1:l-1]$ we have
\begin{multline*}
\left\vert X_k(x_1,x_2,\theta_{1:k-1}) X'_{k+1,l}(x_1,x_2,\theta_{1:l-1}) - \varrho^{\circ (k-1)}(\rho_1(x_1,x_2)) \prod_{k'=k+1}^l \varrho'(\varrho^{\circ (k'-2)}(\rho_1(x_1,x_2))) \right\vert \\
\leq \log(k) \epsilon
+ \Delta_\phi^{-2} (k-1)^{-2} \epsilon 
+ \log\left( \frac{l}{k} \right) \epsilon 
\leq (\Delta_\phi^{-2} (k-1)^{-2} + \log(l)) \epsilon,
\end{multline*}
so that
\begin{multline*}
\left\vert \sum_{k=1}^{l-1} X_k(x_1,x_2,\theta_{1:k-1}) X'_{k+1,l}(x_1,x_2,\theta_{1:l-1}) + X_l(x_1,x_2,\theta_{1:l-1}) 
\right. \\ \left.
- \sum_{k=1}^l \varrho^{\circ (k-1)}(\rho_1(x_1,x_2)) \prod_{k'=k+1}^l \varrho'(\varrho^{\circ (k'-2)}(\rho_1(x_1,x_2))) \right\vert 
\leq (\Delta_\phi^{-2} + \log(l) l) \epsilon.
\end{multline*}

By Proposition~\ref{prop:readout_expectation}, this quantity also bounds
\begin{multline*}
\left\Vert \E_{A_l} K_\theta(x_{i_1},x_{i_2}) - \left( \sum_{k=1}^l \varrho^{\circ (k-1)}(\rho_1(x_1,x_2)) \prod_{k'=k+1}^l \varrho'(\varrho^{\circ (k'-2)}(\rho_1(x_1,x_2))) \right) \Id_{m_l} \right\Vert.
\end{multline*}

Combining \citet[Lemma~2.2.13]{Vandervaartetal2023} with Proposition~\ref{prop:readout_concentration}, for all $t \geq 0$ we have
\begin{multline*}
\Prob_l\left( \max_{i_1,i_2 \in [1:n]}\left\{ \Vert K_\theta(x_{i_1},x_{i_2}) - \E_{A_l} K_\theta(x_{i_1},x_{i_2}) \Vert \right\} 
\right. \\ \left.
\geq O\left( \sqrt{\log(n)} \max_{i_1,i_2 \in [1:n]}\left\{ \Vert J(x_1,x_2,\theta_{1 : l-1}) \Vert_F \right\} m_l^{\frac{1}{2}} \right) t^{\frac{1}{2}}
\right. \\ \left. 
+ O\left( \log(n) \max_{i_1,i_2 \in [1:n]}\left\{ \Vert J(x_1,x_2,\theta_{1 : l-1}) \Vert \right\} m_l \right) t \right)
\leq 2e^{-t}.
\end{multline*}
Condition on the opposite of this event happening as well with $t = m^p$, so that the full probability bound becomes $1-16e^{-m^p}$. Note that for all $i_1,i_2 \in [1:n]$ we have
\begin{multline*}
\Vert J(x_1,x_2,\theta_{1 : l-1}) \Vert_F
\leq \sum_{k=1}^{l-1} O(\Vert B_{k+1, l}(x_1, \theta_{1 : l-1}) B_{k+1, l}(x_2, \theta_{1 : l-1})^* \Vert_F) \\
\leq \sum_{k=1}^{l-1} O(\Vert B_{k+1, l}(x_1, \theta_{1 : l-1}) \Vert \Vert B_{k+1, l}(x_2, \theta_{1 : l-1}) \Vert_F) \\
\leq \sum_{k=1}^{l-1} O\left( \kappa_\phi^2 \log\left( \frac{l-1}{k} \right) m^{-\frac{1}{2}} \right)
\leq O\left( \kappa_\phi^2 (l-1) m^{-\frac{1}{2}} \right)
\end{multline*}
and
\begin{multline*}
\Vert J(x_1,x_2,\theta_{1 : l-1}) \Vert
\leq \sum_{k=1}^{l-1} O(\Vert B_{k+1, l}(x_1, \theta_{1 : l-1}) \Vert \Vert B_{k+1, l}(x_2, \theta_{1 : l-1}) \Vert) \\
\leq \sum_{k=1}^{l-1} O\left( \kappa_\phi^2 \log\left( \frac{l-1}{k} \right) m^{-\frac{1}{2}} \right)^2
\leq O\left( \kappa_\phi^4 (l-1) m^{-1} \right),
\end{multline*}
so that we have
\begin{equation}\label{eq:limiting_readout_concentration}
\Vert K_\theta(x_{i_1},x_{i_2}) - \E_{A_l} K_\theta(x_{i_1},x_{i_2}) \Vert
\leq (l-1) l^{\frac{1}{2}} \epsilon
\end{equation}
for $i_1,i_2 \in [1:n]$. By the triangle inequality, we then have $\Vert K_\theta(x_i,x_i) - l \Id_{m_l} \Vert \leq \log(l) l \epsilon + (l-1) m_l^{\frac{1}{2}} \epsilon = (\log(l) + m_l^{\frac{1}{2}}) l \epsilon$ and 
\begin{multline*}
\left\Vert K_\theta(x_{i_1},x_{i_2}) - \left( \sum_{k=1}^l \varrho^{\circ (k-1)}(\rho_1(x_1,x_2)) \prod_{k'=k+1}^l \varrho'(\varrho^{\circ (k'-2)}(\rho_1(x_1,x_2))) \right) \Id_{m_l} \right\Vert \\
\leq (\Delta_\phi^{-2} + \log(l) l) \epsilon
+ (l-1) m_l^{\frac{1}{2}} \epsilon
= (\Delta_\phi^{-2} + (\log(l) + m_l^{\frac{1}{2}}) l) \epsilon.
\end{multline*}

By \citet[Theorem~1.13.1]{Tretter2008} and \citet[Remark~1.13.2]{Tretter2008}, we then have
\[
\left\Vert K(\theta) - \limiting{K} \right\Vert
\leq \max_{i_1 \in [1:n]}\left\{ \frac{1}{n} \sum_{i_2=1}^n (\Delta_\phi^{-2} + (\log(l) + m_l^{\frac{1}{2}}) l) \epsilon \right\}
\leq (\Delta_\phi^{-2} + (\log(l) + m_l^{\frac{1}{2}}) l) \epsilon.
\]

So far, we have assumed that $\Vert x_i \Vert = 1$ for all $i \in [1:n]$. Note now that by the homogeneity of $\phi$, we have $N(x,\theta) = \Vert x \Vert N(\frac{x}{\Vert x \Vert}, \theta)$ for all $x \in \R^{m_0}$ and $\theta \in \Theta$, so that $\partial_\theta N(x,\theta) = \Vert x \Vert \partial_\theta N(\frac{x}{\Vert x \Vert}, \theta)$ as well. Denoting by $\hat{K}(\theta)$ the NTK matrix over the normalized dataset $\{ \frac{x_i}{\Vert x_i \Vert} : i \in [1:n] \}$, the corresponding limit by $\limiting{\hat{K}}$ and the vector of norms $\tau = [ \Vert x_i \Vert : i \in [1:n] ]$, we then have $K(\theta) = (D_\tau \boxtimes \Id_{m_l}) \hat{K}(\theta) (D_\tau \boxtimes \Id_{m_l})$ and $\limiting{K} = (D_\tau \boxtimes \Id_{m_l}) \limiting{\hat{K}} (D_\tau \boxtimes \Id_{m_l})$, so that we get the claim as
\[
\left\Vert K(\theta) - \limiting{K} \right\Vert
= \left\Vert (D_\tau \boxtimes \Id_{m_l}) (\hat{K}(\theta) - \limiting{\hat{K}}) (D_\tau \boxtimes \Id_{m_l}) \right\Vert
\leq \Vert \tau \Vert_\infty^2 \left\Vert \hat{K}(\theta) - \limiting{\hat{K}} \right\Vert.
\]
\end{proof}

\begin{remark}[Optimal $p$]
Setting
\begin{equation}\label{eq:optimal_p}
p=\log_m(\log(m)),
\end{equation}
we have $m^p = \log(m)$, so that Theorem~\ref{thm:ntk_concentration} gives that
\[
\Prob\left( \left\Vert K(\theta) - \limiting{K} \right\Vert \leq O\left( \overline{\tau}^2 \left( \Delta_\phi^{-2} + \left( \log(l) + m_l^{\frac{1}{2}} \right) l \right) \sqrt{\log(ln) \log(m)} \kappa_\phi^2 m^{-\frac{1}{2}} \right) \right)
\]
is at least $1-O(m^{-1})$.
\end{remark}

\section{Limitations and Future Directions} \label{limitations}
The main limitation of our theory is that even though our MLP parameterization is quite flexible and covers both the kernel and rich regimes, it is still just an MLP, a basic neural network architecture with a narrow range of practical applicability in real-world problems. One future direction is to extend our results to other architectures such as convolutional neural networks and transformers. Another limitation of our work is that even though we proposed a number of hyperparameter settings that are in some sense optimal, we did not provide experimental evaluation of the possible empirical benefits during training. We intend to keep this paper focused on initialization and explore the practical implications in a followup paper. On the purely theoretical side, while our result can readily be applied to study the training of MLPs in the kernel regime by exploiting the lazy training phenomenon, we believe the most important future direction to be the study of the behavior of the NTK matrix during training in the rich regime, where lazy training is absent and the NTK matrix evolves in a nontrivial manner.

\section*{Acknowledgements}
D\'avid Terj\'ek and Diego Gonz\'alez-S\'anchez were supported by the Ministry of Innovation and Technology NRDI Office within the framework of the Artificial Intelligence National Laboratory (RRF-2.3.1-21-2022-00004).

\newpage

\bibliography{jmlr_submission}

\begin{thebibliography}{39}
\providecommand{\natexlab}[1]{#1}
\providecommand{\url}[1]{\texttt{#1}}
\expandafter\ifx\csname urlstyle\endcsname\relax
  \providecommand{\doi}[1]{doi: #1}\else
  \providecommand{\doi}{doi: \begingroup \urlstyle{rm}\Url}\fi

\bibitem[Arora et~al.(2019)Arora, Du, Hu, Li, and Wang]{Aroraetal2019}
Sanjeev Arora, Simon~Shaolei Du, Wei Hu, Zhiyuan Li, and Ruosong Wang.
\newblock Fine-grained analysis of optimization and generalization for
  overparameterized two-layer neural networks.
\newblock In \emph{International Conference on Machine Learning}, 2019.

\bibitem[Banerjee et~al.(2023)Banerjee, Cisneros-Velarde, Zhu, and
  Belkin]{Banerjeeetal2023}
Arindam Banerjee, Pedro Cisneros-Velarde, Libin Zhu, and Misha Belkin.
\newblock Neural tangent kernel at initialization: Linear width suffices.
\newblock In \emph{The 39th Conference on Uncertainty in Artificial
  Intelligence}, 2023.
\newblock URL \url{https://openreview.net/forum?id=VJaoe7Rp9tZ}.

\bibitem[Bombari et~al.(2022)Bombari, Amani, and Mondelli]{Bombarietal2022}
Simone Bombari, Mohammad~Hossein Amani, and Marco Mondelli.
\newblock Memorization and optimization in deep neural networks with minimum
  over-parameterization.
\newblock In Alice~H. Oh, Alekh Agarwal, Danielle Belgrave, and Kyunghyun Cho,
  editors, \emph{Advances in Neural Information Processing Systems}, 2022.
\newblock URL \url{https://openreview.net/forum?id=x8DNliTBSYY}.

\bibitem[Chizat et~al.(2019)Chizat, Oyallon, and Bach]{Chizatetal2019}
L\'{e}na\"{\i}c Chizat, Edouard Oyallon, and Francis Bach.
\newblock On lazy training in differentiable programming.
\newblock In H.~Wallach, H.~Larochelle, A.~Beygelzimer, F.~d\textquotesingle
  Alch\'{e}-Buc, E.~Fox, and R.~Garnett, editors, \emph{Advances in Neural
  Information Processing Systems}, volume~32. Curran Associates, Inc., 2019.
\newblock URL
  \url{https://proceedings.neurips.cc/paper/2019/file/ae614c557843b1df326cb29c57225459-Paper.pdf}.

\bibitem[Daniely et~al.(2016)Daniely, Frostig, and Singer]{Danielyetal2016}
Amit Daniely, Roy Frostig, and Yoram Singer.
\newblock Toward deeper understanding of neural networks: The power of
  initialization and a dual view on expressivity.
\newblock 29, 2016.
\newblock URL
  \url{https://proceedings.neurips.cc/paper_files/paper/2016/file/abea47ba24142ed16b7d8fbf2c740e0d-Paper.pdf}.

\bibitem[Du et~al.(2019{\natexlab{a}})Du, Lee, Li, Wang, and Zhai]{Duetal2019}
Simon~S. Du, Jason Lee, Haochuan Li, Liwei Wang, and Xiyu Zhai.
\newblock Gradient descent finds global minima of deep neural networks.
\newblock In Kamalika Chaudhuri and Ruslan Salakhutdinov, editors,
  \emph{Proceedings of the 36th International Conference on Machine Learning},
  volume~97 of \emph{Proceedings of Machine Learning Research}, pages
  1675--1685. PMLR, 09--15 Jun 2019{\natexlab{a}}.
\newblock URL \url{https://proceedings.mlr.press/v97/du19c.html}.

\bibitem[Du et~al.(2019{\natexlab{b}})Du, Zhai, Poczos, and Singh]{Duetal2018}
Simon~S. Du, Xiyu Zhai, Barnabas Poczos, and Aarti Singh.
\newblock Gradient descent provably optimizes over-parameterized neural
  networks.
\newblock In \emph{International Conference on Learning Representations},
  2019{\natexlab{b}}.
\newblock URL \url{https://openreview.net/forum?id=S1eK3i09YQ}.

\bibitem[Hanin and Nica(2020)]{Haninetal2020}
Boris Hanin and Mihai Nica.
\newblock Finite depth and width corrections to the neural tangent kernel.
\newblock In \emph{International Conference on Learning Representations}, 2020.
\newblock URL \url{https://openreview.net/forum?id=SJgndT4KwB}.

\bibitem[Hayou et~al.(2019)Hayou, Doucet, and Rousseau]{Hayouetal2019}
Soufiane Hayou, Arnaud Doucet, and Judith Rousseau.
\newblock On the impact of the activation function on deep neural networks
  training.
\newblock In Kamalika Chaudhuri and Ruslan Salakhutdinov, editors,
  \emph{Proceedings of the 36th International Conference on Machine Learning},
  volume~97 of \emph{Proceedings of Machine Learning Research}, pages
  2672--2680. PMLR, 09--15 Jun 2019.
\newblock URL \url{https://proceedings.mlr.press/v97/hayou19a.html}.

\bibitem[Hayou et~al.(2022)Hayou, Doucet, and Rousseau]{Hayouetal2022}
Soufiane Hayou, Arnaud Doucet, and Judith Rousseau.
\newblock The curse of depth in kernel regime.
\newblock In Melanie~F. Pradier, Aaron Schein, Stephanie Hyland, Francisco
  J.~R. Ruiz, and Jessica~Z. Forde, editors, \emph{Proceedings on "I (Still)
  Can't Believe It's Not Better!" at NeurIPS 2021 Workshops}, volume 163 of
  \emph{Proceedings of Machine Learning Research}, pages 41--47. PMLR, 13 Dec
  2022.
\newblock URL \url{https://proceedings.mlr.press/v163/hayou22a.html}.

\bibitem[Horn and Mathias(1992)]{Hornetal1992}
Roger~A. Horn and Roy Mathias.
\newblock Block-matrix generalizations of schur's basic theorems on hadamard
  products.
\newblock \emph{Linear Algebra and its Applications}, 172:\penalty0 337--346,
  1992.
\newblock ISSN 0024-3795.
\newblock \doi{https://doi.org/10.1016/0024-3795(92)90033-7}.
\newblock URL
  \url{https://www.sciencedirect.com/science/article/pii/0024379592900337}.

\bibitem[Jacot et~al.(2018)Jacot, Gabriel, and Hongler]{Jacotetal2018}
Arthur Jacot, Franck Gabriel, and Clement Hongler.
\newblock Neural tangent kernel: Convergence and generalization in neural
  networks.
\newblock In S.~Bengio, H.~Wallach, H.~Larochelle, K.~Grauman, N.~Cesa-Bianchi,
  and R.~Garnett, editors, \emph{Advances in Neural Information Processing
  Systems}, volume~31. Curran Associates, Inc., 2018.
\newblock URL
  \url{https://proceedings.neurips.cc/paper/2018/file/5a4be1fa34e62bb8a6ec6b91d2462f5a-Paper.pdf}.

\bibitem[Liu et~al.(2022)Liu, Zhu, and Belkin]{Liuetal2022}
Chaoyue Liu, Libin Zhu, and Mikhail Belkin.
\newblock Loss landscapes and optimization in over-parameterized non-linear
  systems and neural networks.
\newblock \emph{Applied and Computational Harmonic Analysis}, 59:\penalty0
  85--116, 2022.
\newblock ISSN 1063-5203.
\newblock \doi{https://doi.org/10.1016/j.acha.2021.12.009}.
\newblock URL
  \url{https://www.sciencedirect.com/science/article/pii/S106352032100110X}.
\newblock Special Issue on Harmonic Analysis and Machine Learning.

\bibitem[Montanari and Zhong(2022)]{Montanarietal2020}
Andrea Montanari and Yiqiao Zhong.
\newblock {The interpolation phase transition in neural networks: Memorization
  and generalization under lazy training}.
\newblock \emph{The Annals of Statistics}, 50\penalty0 (5):\penalty0 2816 --
  2847, 2022.
\newblock \doi{10.1214/22-AOS2211}.
\newblock URL \url{https://doi.org/10.1214/22-AOS2211}.

\bibitem[Nguyen(2021)]{Nguyen2021}
Quynh~N. Nguyen.
\newblock On the proof of global convergence of gradient descent for deep relu
  networks with linear widths.
\newblock 2021.
\newblock URL \url{https://api.semanticscholar.org/CorpusID:231698350}.

\bibitem[Nguyen and Mondelli(2020)]{Nguyenetal2020}
Quynh~N Nguyen and Marco Mondelli.
\newblock Global convergence of deep networks with one wide layer followed by
  pyramidal topology.
\newblock 33:\penalty0 11961--11972, 2020.
\newblock URL
  \url{https://proceedings.neurips.cc/paper_files/paper/2020/file/8abfe8ac9ec214d68541fcb888c0b4c3-Paper.pdf}.

\bibitem[Nguyen et~al.(2021)Nguyen, Mondelli, and Montufar]{Nguyenetal2021}
Quynh~N. Nguyen, Marco Mondelli, and Guido~F. Montufar.
\newblock Tight bounds on the smallest eigenvalue of the neural tangent kernel
  for deep relu networks.
\newblock In Marina Meila and Tong Zhang, editors, \emph{Proceedings of the
  38th International Conference on Machine Learning}, volume 139 of
  \emph{Proceedings of Machine Learning Research}, pages 8119--8129. PMLR,
  18--24 Jul 2021.
\newblock URL \url{https://proceedings.mlr.press/v139/nguyen21g.html}.

\bibitem[Oymak and Soltanolkotabi(2019)]{Oymaketal2019a}
Samet Oymak and Mahdi Soltanolkotabi.
\newblock Overparameterized nonlinear learning: Gradient descent takes the
  shortest path?
\newblock In Kamalika Chaudhuri and Ruslan Salakhutdinov, editors,
  \emph{Proceedings of the 36th International Conference on Machine Learning},
  volume~97 of \emph{Proceedings of Machine Learning Research}, pages
  4951--4960. PMLR, 09--15 Jun 2019.
\newblock URL \url{https://proceedings.mlr.press/v97/oymak19a.html}.

\bibitem[Oymak and Soltanolkotabi(2020)]{Oymaketal2020}
Samet Oymak and Mahdi Soltanolkotabi.
\newblock Toward moderate overparameterization: Global convergence guarantees
  for training shallow neural networks.
\newblock \emph{IEEE Journal on Selected Areas in Information Theory},
  1\penalty0 (1):\penalty0 84--105, 2020.
\newblock \doi{10.1109/JSAIT.2020.2991332}.

\bibitem[Poole et~al.(2016)Poole, Lahiri, Raghu, Sohl-Dickstein, and
  Ganguli]{Pooleetal2016}
Ben Poole, Subhaneil Lahiri, Maithra Raghu, Jascha Sohl-Dickstein, and Surya
  Ganguli.
\newblock Exponential expressivity in deep neural networks through transient
  chaos.
\newblock In D.~Lee, M.~Sugiyama, U.~Luxburg, I.~Guyon, and R.~Garnett,
  editors, \emph{Advances in Neural Information Processing Systems}, volume~29.
  Curran Associates, Inc., 2016.
\newblock URL
  \url{https://proceedings.neurips.cc/paper_files/paper/2016/file/148510031349642de5ca0c544f31b2ef-Paper.pdf}.

\bibitem[Schoenholz et~al.(2017)Schoenholz, Gilmer, Ganguli, and
  Sohl-Dickstein]{Schoenholzetal2017}
Samuel~S. Schoenholz, Justin Gilmer, Surya Ganguli, and Jascha Sohl-Dickstein.
\newblock Deep information propagation.
\newblock In \emph{International Conference on Learning Representations}, 2017.
\newblock URL \url{https://openreview.net/forum?id=H1W1UN9gg}.

\bibitem[Seleznova and Kutyniok(2022)]{Seleznovaetal22}
Mariia Seleznova and Gitta Kutyniok.
\newblock Neural tangent kernel beyond the infinite-width limit: Effects of
  depth and initialization.
\newblock In Kamalika Chaudhuri, Stefanie Jegelka, Le~Song, Csaba Szepesvari,
  Gang Niu, and Sivan Sabato, editors, \emph{Proceedings of the 39th
  International Conference on Machine Learning}, volume 162 of
  \emph{Proceedings of Machine Learning Research}, pages 19522--19560. PMLR,
  17--23 Jul 2022.
\newblock URL \url{https://proceedings.mlr.press/v162/seleznova22a.html}.

\bibitem[Song et~al.(2021)Song, Ramezani-Kebrya, Pethick, Eftekhari, and
  Cevher]{Songetal2021}
Chaehwan Song, Ali Ramezani-Kebrya, Thomas Pethick, Armin Eftekhari, and Volkan
  Cevher.
\newblock Subquadratic overparameterization for shallow neural networks.
\newblock In A.~Beygelzimer, Y.~Dauphin, P.~Liang, and J.~Wortman Vaughan,
  editors, \emph{Advances in Neural Information Processing Systems}, 2021.
\newblock URL \url{https://openreview.net/forum?id=NhbFhfM960}.

\bibitem[Su and Yang(2019)]{Suetal2019}
Lili Su and Pengkun Yang.
\newblock On learning over-parameterized neural networks: A functional
  approximation perspective.
\newblock In H.~Wallach, H.~Larochelle, A.~Beygelzimer, F.~d\textquotesingle
  Alch\'{e}-Buc, E.~Fox, and R.~Garnett, editors, \emph{Advances in Neural
  Information Processing Systems}, volume~32. Curran Associates, Inc., 2019.
\newblock URL
  \url{https://proceedings.neurips.cc/paper_files/paper/2019/file/253f7b5d921338af34da817c00f42753-Paper.pdf}.

\bibitem[Terj\'ek and Gonz\'alez-S\'anchez(2025)]{mlpsateoc1}
D\'avid Terj\'ek and Diego Gonz\'alez-S\'anchez.
\newblock {MLP}s at the {EOC}: Spectrum of the {NTK}, 2025.

\bibitem[Tretter(2008)]{Tretter2008}
Christiane Tretter.
\newblock \emph{Spectral Theory of Block Operator Matrices and Applications}.
\newblock IMPERIAL COLLEGE PRESS, 2008.
\newblock \doi{10.1142/p493}.
\newblock URL \url{https://www.worldscientific.com/doi/abs/10.1142/p493}.

\bibitem[van~der Vaart and Wellner(2023)]{Vandervaartetal2023}
A.W. van~der Vaart and J.A. Wellner.
\newblock \emph{Weak Convergence and Empirical Processes: With Applications to
  Statistics}.
\newblock Springer Series in Statistics. Springer International Publishing,
  2023.
\newblock ISBN 9783031290404.
\newblock URL \url{https://books.google.hu/books?id=vfzKEAAAQBAJ}.

\bibitem[Vershynin(2018)]{Vershynin2018}
Roman Vershynin.
\newblock \emph{High-Dimensional Probability: An Introduction with Applications
  in Data Science}.
\newblock Cambridge Series in Statistical and Probabilistic Mathematics.
  Cambridge University Press, 2018.
\newblock \doi{10.1017/9781108231596}.

\bibitem[Wang and Zhu(2024)]{Wangetal2021}
Zhichao Wang and Yizhe Zhu.
\newblock {Deformed semicircle law and concentration of nonlinear random
  matrices for ultra-wide neural networks}.
\newblock \emph{The Annals of Applied Probability}, 34\penalty0 (2):\penalty0
  1896 -- 1947, 2024.
\newblock \doi{10.1214/23-AAP2010}.
\newblock URL \url{https://doi.org/10.1214/23-AAP2010}.

\bibitem[Woodworth et~al.(2020)Woodworth, Gunasekar, Lee, Moroshko, Savarese,
  Golan, Soudry, and Srebro]{Woodworthetal2020}
Blake Woodworth, Suriya Gunasekar, Jason~D. Lee, Edward Moroshko, Pedro
  Savarese, Itay Golan, Daniel Soudry, and Nathan Srebro.
\newblock Kernel and rich regimes in overparametrized models.
\newblock In Jacob Abernethy and Shivani Agarwal, editors, \emph{Proceedings of
  Thirty Third Conference on Learning Theory}, volume 125 of \emph{Proceedings
  of Machine Learning Research}, pages 3635--3673. PMLR, 09--12 Jul 2020.
\newblock URL \url{https://proceedings.mlr.press/v125/woodworth20a.html}.

\bibitem[Xiao et~al.(2020)Xiao, Pennington, and Schoenholz]{Xiaoetal2020}
Lechao Xiao, Jeffrey Pennington, and Samuel Schoenholz.
\newblock Disentangling trainability and generalization in deep neural
  networks.
\newblock In Hal~Daumé III and Aarti Singh, editors, \emph{Proceedings of the
  37th International Conference on Machine Learning}, volume 119 of
  \emph{Proceedings of Machine Learning Research}, pages 10462--10472. PMLR,
  13--18 Jul 2020.
\newblock URL \url{https://proceedings.mlr.press/v119/xiao20b.html}.

\bibitem[Xu and Zhu(2024)]{Xuetal2024}
Jiaming Xu and Hanjing Zhu.
\newblock Overparametrized multi-layer neural networks: Uniform concentration
  of neural tangent kernel and convergence of stochastic gradient descent.
\newblock \emph{Journal of Machine Learning Research}, 25\penalty0
  (94):\penalty0 1--83, 2024.
\newblock URL \url{http://jmlr.org/papers/v25/23-0740.html}.

\bibitem[Yang(2020)]{Yang2020}
Greg Yang.
\newblock Tensor programs ii: Neural tangent kernel for any architecture, 2020.

\bibitem[Yang(2021)]{Yang2021}
Greg Yang.
\newblock Tensor programs iii: Neural matrix laws, 2021.

\bibitem[Yang and Hu(2021)]{Yangetal2021}
Greg Yang and Edward~J. Hu.
\newblock Tensor programs iv: Feature learning in infinite-width neural
  networks.
\newblock In Marina Meila and Tong Zhang, editors, \emph{Proceedings of the
  38th International Conference on Machine Learning}, volume 139 of
  \emph{Proceedings of Machine Learning Research}, pages 11727--11737. PMLR,
  18--24 Jul 2021.
\newblock URL \url{https://proceedings.mlr.press/v139/yang21c.html}.

\bibitem[Yang et~al.(2023)Yang, Simon, and Bernstein]{Yangetal2023}
Greg Yang, James~B. Simon, and Jeremy Bernstein.
\newblock A spectral condition for feature learning, 2023.

\bibitem[Yang et~al.(2024{\natexlab{a}})Yang, Hu, Babuschkin, Sidor, Liu,
  Farhi, Ryder, Pachocki, Chen, and Gao]{Yangetal2022}
Greg Yang, Edward~J. Hu, Igor Babuschkin, Szymon Sidor, Xiaodong Liu, David
  Farhi, Nick Ryder, Jakub Pachocki, Weizhu Chen, and Jianfeng Gao.
\newblock Tensor programs v: tuning large neural networks via zero-shot
  hyperparameter transfer.
\newblock In \emph{Proceedings of the 35th International Conference on Neural
  Information Processing Systems}, NIPS '21, Red Hook, NY, USA,
  2024{\natexlab{a}}. Curran Associates Inc.
\newblock ISBN 9781713845393.

\bibitem[Yang et~al.(2024{\natexlab{b}})Yang, Yu, Zhu, and Hayou]{Yangetal2024}
Greg Yang, Dingli Yu, Chen Zhu, and Soufiane Hayou.
\newblock Tensor programs {VI}: Feature learning in infinite depth neural
  networks.
\newblock In \emph{The Twelfth International Conference on Learning
  Representations}, 2024{\natexlab{b}}.
\newblock URL \url{https://openreview.net/forum?id=17pVDnpwwl}.

\bibitem[Zou and Gu(2019)]{Zouetal2019}
Difan Zou and Quanquan Gu.
\newblock An improved analysis of training over-parameterized deep neural
  networks.
\newblock 32, 2019.
\newblock URL
  \url{https://proceedings.neurips.cc/paper_files/paper/2019/file/6a61d423d02a1c56250dc23ae7ff12f3-Paper.pdf}.

\end{thebibliography}

\end{document}